\documentclass{article}

\usepackage{microtype}
\usepackage{graphicx}
\usepackage{subfigure}
\usepackage{booktabs} % for professional tables

\usepackage{hyperref}

\usepackage[accepted]{icml2021}

\icmltitlerunning{Improved Corruption Robust Algorithms for Episodic Reinforcement Learning}

\usepackage{graphicx}
\usepackage{mathtools}
\usepackage{footnote}
\usepackage{float}
\usepackage{xspace}
\usepackage{multirow}
\usepackage{wrapfig}
\usepackage{framed}
\usepackage{amsfonts}

\newcommand{\calA}{{\mathcal{A}}}

\newcommand{\calS}{{\mathcal{S}}}

\newcommand{\calI}{{\mathcal{I}}}

\newcommand{\calD}{{\mathcal{D}}}
\newcommand{\calE}{{\mathcal{E}}}

\newcommand{\calM}{{\mathcal{M}}}

\newcommand{\Reg}{\text{\rm Reg}}
\newcommand{\one}{\boldsymbol{1}}

\newcommand{\field}[1]{\mathbb{#1}}
\newcommand{\E}{\field{E}}

\newcommand{\argmax}{\text{argmax}}
\newcommand{\prob}{\text{Prob}}

\newcommand{\order}{\ensuremath{\mathcal{O}}}

\usepackage{amsthm}
\newtheorem{theorem}{Theorem}
\newtheorem{lemma}{Lemma}
\newtheorem{cor}{Corollary}
\newtheorem{remark}{Remark}

\newtheorem{definition}{Definition}

\newcommand{\yf}[1]{{\color{red}\bf[YF: #1]}}
\newcommand{\kevin}[1]{{\color{magenta}\bf[KJ: #1]}}

\begin{document}

\twocolumn[
\icmltitle{Improved Corruption Robust Algorithms for Episodic Reinforcement Learning}

% It is OKAY to include author information, even for blind
% submissions: the style file will automatically remove it for you
% unless you've provided the [accepted] option to the icml2021
% package.

% List of affiliations: The first argument should be a (short)
% identifier you will use later to specify author affiliations
% Academic affiliations should list Department, University, City, Region, Country
% Industry affiliations should list Company, City, Region, Country

% You can specify symbols, otherwise they are numbered in order.
% Ideally, you should not use this facility. Affiliations will be numbered
% in order of appearance and this is the preferred way.
% \icmlsetsymbol{equal}{*}

\begin{icmlauthorlist}
\icmlauthor{Yifang Chen}{equal}
\icmlauthor{Simon S. Du}{equal}
\icmlauthor{Kevin Jamieson}{equal}
\end{icmlauthorlist}

\icmlaffiliation{equal}{Paul G. Allen School of Computer Science \& Engineering,
University of Washington}

\icmlcorrespondingauthor{Yifang Chen}{yifangc@cs.washington.edu}

% You may provide any keywords that you
% find helpful for describing your paper; these are used to populate
% the "keywords" metadata in the PDF but will not be shown in the document
\icmlkeywords{Machine Learning, ICML}

\vskip 0.3in
]

% this must go after the closing bracket ] following \twocolumn[ ...

% This command actually creates the footnote in the first column
% listing the affiliations and the copyright notice.
% The command takes one argument, which is text to display at the start of the footnote.
% The \icmlEqualContribution command is standard text for equal contribution.
% Remove it (just {}) if you do not need this facility.

\printAffiliationsAndNotice{}  % leave blank if no need to mention equal contribution
% \printAffiliationsAndNotice{\icmlEqualContribution} % otherwise use the standard text.
% \theHalgorithm
\begin{abstract}
    We study episodic reinforcement learning under unknown adversarial corruptions in both the rewards and the transition probabilities of the underlying system.
    We propose new algorithms which, compared to the existing results in \cite{lykouris2020corruption}, achieve strictly better regret bounds in terms of total corruptions for the tabular setting. To be specific, firstly, our regret bounds depend on more precise numerical values of total rewards corruptions and transition corruptions, instead of only on the total number of corrupted episodes. Secondly, our regret bounds are the first of their kind in the reinforcement learning setting to have the number of corruptions show up additively with respect to $\min\{ \sqrt{T},\text{PolicyGapComplexity} \}$ rather than multiplicatively. 
    % regrets are the first one who get corruptions only appears in additerms instead of multiplicative terms.
    %
    % In addition to the our results on regret,
    Our results follow from a general algorithmic framework that combines corruption-robust policy elimination meta-algorithms, and plug-in reward-free exploration sub-algorithms. 
    Replacing the meta-algorithm or sub-algorithm may extend the framework to address other corrupted settings with potentially more structure.
    % This opens the possibility of  replacing the meta-algorithm with any elimination based multi-arm (or linear) bandits algorithms, and replacing the sub-algorithms with more efficient or robust exploration reinforcement learning algorithms, in order to achieve better results in corrupted reinforcement related problems.
\end{abstract}
\section{Introduction}
\label{sec: intro}

% \yf{
% Intro pending. Remember to briefly introduce the $C^p,C^r$ notation. For results that will be added here or in related paper, see the comments on ref.bib
% }
% 
Reinforcement learning (RL) studies the problem where the learner interacts with the environment sequentially and aims to improve its decision making strategy over time. This problem has usually been modelled as a Markov Decision Process (MDP) with unknown transition functions. In this paper, we consider the classical episodic reinforcement learning with a finite horizon. 
Within each episode, the learner sequentially observes the current state at each stage, plays an action, receives the reward according to the current state-action pair, and then transitions to the next stage according to the underlying transition function. 

The majority of the literature in learning in MDPs studies stationary environments, where the underlying unknown transition function and reward function are fixed. The rewards and the next states are independently and identically distributed given the current state and the learner’s chosen action. Under this setting, the goal is to minimize the regret, which is the difference between the learner’s cumulative rewards and the total rewards of the optimal policy~\citep{brafman2002r,azar2017minimax,jin2018q,ok2018exploration,zanette2019tighter,DBLP:conf/nips/SimchowitzJ19,zhang2020reinforcement}. However, these techniques are vulnerable to corruptions on the rewards or the transitions. Recently, \citet{rosenberg2019online,jin2020learning,chung-wei2020bias} gave provably efficient algorithms for the setting of adversarial rewards and fixed unknown transitions. 
% Their regret is defined only in expectation of transition functions accordingly. 
Although their algorithms are robust to corruptions on rewards, they heavily rely on the assumption that the transitions are not corrupted.

The most relevant work is by \citet{lykouris2020corruption} who gave the first set of results on episodic reinforcement learning that achieve robustness to corruptions on both the rewards and the transition functions. Their regret is defined as the difference between the learner’s accumulated rewards and the total rewards of the optimal fixed policy with respect to the \textit{uncorrupted} underlying rewards and transition functions. Their algorithm is efficient and works for tabular RL and its linear variants. Unfortunately, their algorithm is not optimal in terms of the corruption level. 
Firstly, their corruption level $C$ is defined as the total number of corrupted episodes.
% \simon{Use $C$ for corruption level instead of corruption..}
Ideally, we would like the regret to depend on more fine-grained characterizations of corruptions such as the total magnitude of corruptions on the rewards ($C^r$) and transition functions ($C^p$). 
Secondly, their regret bound scales $\tilde{\order}\left(C\sqrt{T} + C^2 \right) $ in the worst case, where the corruption level $C$ appears \emph{both additively and multiplicatively}. They state in the paper that it is unclear whether one can obtain additive dependence alone in tabular RL. 
In this paper, we address this open problem.

\paragraph{Our contribution: } To the best of our knowledge, this is the first work for the episodic tabular RL setting that obtains a regret bound that scales only \emph{additively} with respect to the number of corruptions. This result is significant because it demonstrates that a learner can be highly robust to the corruptions, even though the magnitude and number of corrupted episodes are unknown to the learner.
Our detailed contributions are shown as follows. Note that we omit all $\calS,\calA,H$ dependence for clarity.
\begin{itemize}
    \item We first propose a corruption robust reward-free exploration algorithm \textsc{EstAll} such that for a given $\epsilon >0$,  \textsc{EstAll} returns $\left(\epsilon + (C^p + C^r)\epsilon^2 \right)$-close estimations for all policies within a given policy set $\Pi$.
    If the total magnitude of corruptions to the transition functions satisfies $C^p \leq \tilde{\order}(1/\epsilon)$ then the algorithm requires a sample complexity of just $\tilde{\order}(\log|\Pi|/\epsilon^2)$. 
    On the other hand, if $C^p > \tilde{\order}(1/\epsilon)$ then the algorithm will fail to complete within the expected sample complexity, providing the learner with a lower bound on the level of corruptions. 
    % Note that corruptions are even allowed to be given after the adversary seeing learner's chosen action at each stage.
    %  For "partially robust", it means that, the algorithm only requires $\tilde{\order}(\log|\Pi|/\epsilon^2)$ sample complexity, as long as the total corruption of transition functions satisfies $C^p \leq \tilde{\order}(\epsilon)$. 
    % In the other word, the learner can conclude that the total corruption of transition functions $C^p > \tilde{\order}(\epsilon)$ when more other expected number of samples are needed. 
    % In this case, the learner can do some intervention, for example, discard the heavily corrupted data and simply restart the exploration procedure given that the total corruption budget is limited. 
    % \simon{This sentence may not be needed.}
    %
    \item We propose two meta-algorithms for RL inspired by the corruption robust algorithms for multi-armed bandits~\citep{gupta2019better,bogunovic2020stochastic}, both of which use \textsc{EstAll} as a sub-routine.
    % We leverage \textsc{EstAll} as a sub-routine in two proposed meta-algorithms based on the corruption-robust algorithms proposed in \cite{gupta2019better} and \cite{bogunovic2020stochastic} for multi-armed bandits.
    The first meta-algorithm \textsc{BARBAR-RL} guarantees an $\tilde{\order}\big(\min\{\sqrt{T},\text{PolicyGapComplexity} \}+ (1+C^p)(C^p+C^r)\big)$ regret when the adversary must decide whether to corrupt the episode before seeing the learner's chosen deterministic policy at the current episode. 
    The second meta-algorithm \textsc{BrutePolicyElimination-RL} guarantees an $\tilde{\order}\left(\sqrt{T}+ (C^p+C^r)^2\right)$ regret when the adaptive adversary can decide when and how much to corrupt the episode after seeing the learner's chosen action and deterministic policy at each stage of the current episode.\footnote{This is a stronger adversary than the one studied in \citet{lykouris2020corruption}.}

    \item Finally, comparing with \cite{lykouris2020corruption} who defined the corruption level as the total number of corrupted episodes, our bounds depend on much finer definitions based on the magnitudes of corruptions on the reward and the transition ($C^r$ and $C^p$).
    % In addition, the separated dependence on $C^r$ and $C^p$ aligns with our intuition that the difficulty from corrupted RL is mainly from corrupted transition functions, rather than reward functions. 
    % \simon{Are we sure...? Maybe we can say when $C^p$ then one bound only scales linearly with $C^r$ which matches \cite{jin2020learning,rosenberg2019online}.} \yf{I don't think they exactly match? so I just remove the sentence.}
\end{itemize}

\paragraph{Related Work:}
In addition to worst-case $\sqrt{T}$ dependent regret, \citet{lykouris2020corruption} also achieves an instance-dependent bound in terms of GapComplexity for tabular RL by using the UCB type algorithm and the analysis techniques developed in \citet{DBLP:conf/nips/SimchowitzJ19}. It remains unclear whether non-UCB type algorithms, for example, policy-elimination type methods, can also achieve the instance-dependent bound.

Other than the instance-dependent bounds, our regret bounds' dependency on $|\calS|,|\calA|$ and $H$ are not optimal compared to the existing works including \cite{azar2017minimax,jin2018q,ok2018exploration,zanette2019tighter,zhang2020reinforcement}. Whether their techniques can be used in our framework or our policy-elimination-based methods require an entirely different analysis remains unclear.

While the literature on corrupted RL is limited, the corruption robust algorithms have been well studied in multi-arm bandits (MAB) settings, which is a special case of episodic tabular reinforcement learning with horizon $H=1$.
Corrupted MAB problems are relatively simpler than corrupted RL because we are no longer required to deal with the corruption on transition functions. 
In the MAB setting, obtaining a $\sqrt{T}$ regret bound with some $C$ dependence terms, applying either additively or multiplicatively, is quite easily obtained by appealing to algorithms from the adversarial bandits literature such as the classical EXP-3 algorithm \cite{auer2002nonstochastic} that can achieve $\tilde{\order}(\sqrt{T})$ for adversarial rewards. 
Therefore, the majority of works in the corrupted MAB setting seek a $\Delta_a$-dependent regret which scales only logarithmically with $T$, where $\Delta_a$ is the gap between the expected reward of action $a$ and the optimal arm. Despite the simplified setting of corrupted MAB relative to RL, many of the techniques used in those works still provide inspiration for corrupted RL problems.

% The main target and difficulties met in MAB setting is relatively different from the one in reinforcement learning. Because MAB doesn’t require an estimation on transition functions, the algorithm in MAB setting only needs to be robust to the corruption of rewards. Therefore, achieving a $\sqrt{T}$ dependent regret becomes trivial given that the optimal regret for the adversarial rewards always give $\sqrt{T}$. Therefore, most of the MAB algorithms are targeted to a gap-dependent regret. Many of the techniques used in those works are inspiring for the reinforcement learning setting. 
We will briefly review the most relevant corrupted MAB works here. \citet{lykouris2018stochastic} achieves a $\tilde{\order}\left(\sum_{a\neq a^*} \frac{C K}{\Delta_a}\right)$ regret bound by using the multi-layer active arm elimination. \citet{lykouris2020corruption}’s corrupted RL work referenced above is built upon this technique. \citet{gupta2019better} achieves $\tilde{\order}(\sum_{a\neq a^*} \frac{1}{\Delta_a}+KC)$ by adopting a sampling strategy based on the estimated gap instead of eliminating arms permanently. One of our results is built on this technique by regarding each policy as an arm. Finally, \citet{zimmert2019optimal} achieves a near-optimal result $\tilde{\order}\left(\sum_{a\neq a^*} \frac{1}{\Delta_a}+\sqrt{\sum_{a\neq a^*} \frac{C}{\Delta_a}}\right)$ by using Follow-the-Regularized Leader with Tsallis Entropy. Note that their work actually solves a more difficult problem called best-of-both-worlds, which can achieve near-optimal result simultaneously for both adversarial and stochastic rewards.  The similar technique has been adopted in \citet{jin2020simultaneously}, which achieves $\tilde{\order}( \text{GapComplexity} +  \sqrt{C \cdot \text{GapComplexity}})$ when the transition function is known. Unfortunately, whether it is possible to extend such techniques to the unknown transition setting remains unclear. Besides the corrupted MAB setting, \citet{li2019stochastic} also consider linear bandits which achieves a similar result to \citet{lykouris2018stochastic} and \citet{lykouris2020corruption}.

Note that all of these works presented above consider a weak adversary which must decide the corruption for each round (or episodes) before observing the learner's chosen action (or policy). Some works (e.g. \citep{liu2019data,bogunovic2020stochastic}) consider a stronger adversary which can decide the corruption after seeing the learner's current behavior. In particular, \citet{bogunovic2020stochastic} achieves a near-optimal regret $\tilde{\order}(\sqrt{dT} + Cd^{3/2}+C^2)$ for linear bandits by using arm elimination with an enlarged confidence bound. One of our results also considers this stronger adversary setting and adopts a similar technique.

Finally, our reward-free exploration sub-algorithm is based on the algorithm in \citet{wang2020long} by again using the trajectory synthesis idea. But just as in the original algorithm, this exploration sub-algorithm is inefficient. Algorithms proposed in \citet{kaufmann2020adaptive} and \citet{menard2020fast} can efficiently achieve an $\epsilon$-close estimation for each policy given a policy set $\Pi$ when \textit{no corruption} exists. But whether this type of algorithm can be made robust to corruptions at least as good as \textsc{EstAll} remains unknown. We provide some discussion in Appendix~\ref{sec: app-RFalgo}.

\paragraph{Structure of the paper: }
In Section~\ref{sec: prelim}, we formally define our settings and the regret objective. In Section~\ref{sec: mainAlgo}, we describe the meta-algorithm \textsc{BARBAR-RL} for the non-cheated adversary and show a sketch analysis. We also briefly state the \textsc{BrutePolicyElimination-RL} algorithm and its result, postponing the details into the Appendix~\ref{sec: app-mainAlgo2} because it essentially uses the same key techniques as ones in \textsc{BARBAR-RL} analysis. In Section~\ref{sec: subAlgo}, we give a formal description of the reward-free exploration algorithm \textsc{EstAll} as well as its sketch analysis. 
\section{Preliminaries}
\label{sec: prelim}

\paragraph{Episodic reinforcement learning.} Let $\calM = (\calS,\calA, P, R, H, s_1)$ be an episodic \textit {Markov Decision Process (MDP)} where $\calS$ is the finite state space, $\calA$ is the finite action space, $P: \calS \times \calA \times [H] \rightarrow \Delta(\calS)$ is the transition operator which takes a state-action-step pair and returns a distribution over states, $R: \calS \times \calA \rightarrow \Delta(\mathbb{R})$ is the reward distribution and the $H$ is the episodic length. For convenience, we assume that the trajectory always starts from a single state $s_0$, that is, $P(s_1 = s) = 0$ for all $s \neq s_0$. It can be reduced from more general setting by adding an arbitrary starting state.

We have total $T$ episodes. At each episode $t \in [T]$, a deterministic non-stationary policy $\pi$ chooses an action $a \in \calA$ based on the current state $s \in \calS$ and the step $h \in [H]$. Formally, $\pi = \{\pi_h\}_{h=1}^H$ where for each $h \in [H]$, $\pi_h: \calS \rightarrow \calA $ maps a given state to an action. The policy $\pi$ induces a random trajectory $s_1,a_1,r_1,s_2,a_2,r_2,\ldots,s_H,a_H,r_H,s_{H+1}$ where $a_1=\pi_1(s_1),r_1 \sim R(s_1,a_1), s_2 \sim P(\cdot|s_1,a_1,1), a_2=\pi_2(s_2),r_2 \sim R(s_2,a_2),\ldots, a_H=\pi_H(s_H),r_H \sim R(s_H,a_H), s_{H+1} \sim P(\cdot|s_H,a_H,H)$. We define the set of all possible policies as $\Pi = \calA^{\calS \times [H]}$.

Finally, we assume the bounded total reward that $r_h \geq 0$ for all $h \in [H]$ and $\sum_{h=1}^H r_h \in [0,H]$. 
% Note that this is the same reward setting as in \cite{wang2020long}, which is weaker than the standard reward assumption where the immediate rewards $r_h$ are bounded. \simon{Since our $H$-dependence is not tight, we don't need to emphasize we are more genereal...}

\paragraph{Episodic RL with corruption.} When \textit{no corruption} happens, all the samples are consistently generated by a \textit{nominal} MDP $\calM^* = (\calS,\calA, P^*, R^*, H, s_1)$. Here we assume the MDP is stationary, that is  $P(\cdot|s,a,h)= P(\cdot|s,a,h'),R(s,a,h)= R(s,a,h')$ for all $h,h' \in [H]$. 
\\\\
In the \textit{corrupted} setting, before episode $t$, the adversary decides whether to corrupt the episode, in which case the corresponding MDP $\calM_t = (\calS,\calA, P_t, R_t, H, s_1)$ can be arbitrary. Notice that although the \textit{nominal} MDP $\calM^*$ is a stationary MDP, we generally allow the corrupted $\calM_t$ to be non-stationary. We define the corruption numerically at episode $t$ as 
\begin{align*}
    & c_t^r = \sum_{h=2}^H \sup_{(s,a) \in \calS \times \calA}|R_t(s,a,h) - R^*(s,a)| \\
        &\quad + \sup_{a\in \calA} |R_t(s_0,a,1) - R^*(s_0,a)| \\
    & c_t^p = \sum_{h=2}^H \sup_{(s,a) \in \calS \times \calA} \|P_t(\cdot|s,a,h) - P^*(\cdot|s,a)\|_1 \\
        &\quad + \sup_{a\in \calA} \|P_t(\cdot|s_0,a,1) - P^*(\cdot|s_0,a)\|_1
\end{align*}
Notice that we define the corruption on transition and rewards separately because the main difficulty in RL setting comes from corruptions on the transition function. Also, compared to the corruption definition in \citet{lykouris2020corruption}, which merely captures whether an episode has been corrupted or not, our definition is based on the real-valued magnitude of the corruption. Finally, both $\calM^*$ and $\calM_t$, as well as the corruption levels $c_t^p,c_t^r$ are unknown to learner. The adversary can always adaptively decide to corrupt the current episode based on the learner's strategy and the observable history of the previous episode from $1$ to $t-1$, which is the same setting as in \citet{lykouris2020corruption}. But the adversary can be even stronger, that is, it can decide corruption $c_t^p,c_t^r$ after seeing learner's \textit{chosen policy} in each episode or even seeing learner's \textit{state and chosen action} at each stage in each episode.  Here we called it ``cheated adversary". Otherwise, we call it ``non-cheated adversary" for adversary who decides corruption before seeing learner's chosen deterministic policy.

\paragraph{Other Conventions and Notations.} We use the superscript $rp$ as a shorthand to suggest a term holds for both reward and transition corruptions simultaneously. We define the total corruption for any time interval $\calI$ as $C_\calI^{rp} = \sum_{t \in \calI} c_t^{rp}$ and simply denote $C_{[0,T]}^{rp}$ as  $C^{rp}$. For any policy $\pi$, we write the value function under $\calM$ as $V^{\calM,\pi}(s_1)$, and denote $V^{\calM_t,\pi}(s_1)$ as $V_t^\pi(s_1)$, $V^{\calM^*,\pi}(s_1)$ as $V_*^\pi(s_1)$. Also we denote $V^*(s_1) = \max_{\pi \in \Pi}  V_*^{\pi}(s_1)$ and $\Delta_\pi =V^* - V_*^{\pi}(s_1) $. Because we assume a deterministic start state, in the remainder of the paper we omit $s_1$.\\

% \paragraph{Policy Gap.} As stated in the \textbf{Related Work} section, the most common GapComplexity used in reinforcement learning is in the following form.
% \begin{align*}
%     gap_h(s,a) = V_h^*(s) - Q^*_h(s,a).
% \end{align*}
% Here we instead consider the policy gap $\Delta$ as defined above. Consider two policies $\pi$ and $\pi'$ which behave exactly the same from until some $h$ step. Then it is easy to see that 
% \begin{align*}
%     &|V_*^\pi - V_*^{\pi'}| \\
%     &\geq \sum_{s \in \calS} \text{Prob(polices visit $s$ at $h$)} \left(Q^*_h(s,\pi(s)) - Q^*_h(s,\pi'(s)\right)
% \end{align*}

\paragraph{Regret.} 
In this paper, we will focus on the the regret that is only evaluated on the nominal MDP, defined as following,
\begin{align*}
    \Reg_T: =  \sum_{t=1}^T  V^* - V_*^{\pi_t}
\end{align*}
This is the same definition as in \cite{lykouris2020corruption}.

\paragraph{An $\epsilon$-net for Policies.}
Using the same idea as in Section 5.1 of \citet{wang2020long}, we can construct an $\epsilon$-net of non-stationary policies, denoted as $\Pi_\epsilon$.  As proved in their work, $\Pi_\epsilon$ satisfied the following properties
\begin{align}
    & |\Pi_\epsilon| \leq \min\{ (H/\epsilon+1)^{|\calS|^2|\calA| + |\calS||\calA|}, |\calA|^{H|\calS|}\}\\
    & V^* - \max_{\pi \in \Pi_\epsilon} V_*^\pi \leq 8H^2|\calS|\epsilon 
\end{align}
The first property enables us to reduce the sample complexity when $H \gg |\calA|,|\calS|$. The second property ensures that, as long as $\epsilon$ is small enough, the best policy inside  $\Pi_\epsilon$ is close to the true optimal policy. In the remainder of the paper, we will only consider policies inside  $\Pi_{1/T}$ instead of the whole policy space $\calA^{\calS \times [H]}$.

% As stated in \cite{lykouris2020corruption}, a strong dynamic regret, which compares our choice to the best policy at each time, can be upper bounded as following
% \begin{align*}
%     \sum_{t=1}^T \max_{\pi \in \Pi} V_t^{\pi} - V_t^{\pi_t} 
%     \leq \max_{\pi \in \Pi} \sum_{t=1}^T   V_*^{\pi} - V_*^{\pi_t} + C_{T}^p + C_{T}^r
% \end{align*}
% Note that the $C_{T}^p + C_{T}^r$ term is inevitable. \yf{because of the different definition of corruption, add a footnote to say that why this holds will be argued in the appendix} Therefore, in the remainder of the paper, we will focus on the the regret that is only evaluated on the nominal MDP, defined as following,
% \begin{align*}
%     \overline{\Reg}: =  \sum_{t=1}^T  V^* - V_*^{\pi_t}
% \end{align*}
\section{Main Algorithms and Results }
\label{sec: mainAlgo}

We present two algorithms: the first for the non-cheated and the second for the cheated.
Recall that the difference between these settings is the strength of the adversary. The non-cheated must decide the corruption before seeing the learner's current action (or chosen policy) while the cheated can decide afterwards. Thus, for the more challenging setting of a cheated adversary, we expect a larger regret bound. 

\subsection{The Algorithm and the Result for Non-cheated Adversary}

\begin{algorithm}[!t] 
\caption{BARBAR-RL}
\begin{algorithmic}[1] 
\label{alg:main_1}
\STATE \textbf{Input:} time horizon $T$, confidence $\delta_{overall}$
\STATE Construct a $1/T$-net for non-stationary policies, denoted as $\Pi_{1/T}$. 
\STATE Initialize $S_1 = \{0\} , \Pi_0^1 = \Pi_{1/T}$. And for $j \in [\log T]$, initialize $\epsilon_j = 2^{-j}. \epsilon_{est}^j = \epsilon_j/128$
% \yf{|\calA||\calS|H?}
\STATE Set $\lambda_1=6 |\calS||\calA| \log(H^2|\calS||\calA|/\epsilon_{est}) $ and $\lambda_2= 12 \ln(8T/\delta_{overall})$ 
% \yf{add 2 base log and natural log}
\FOR{ epoch $m = 1,2,\ldots$}
    \STATE Set $\delta_j^m = (|\Pi_j^m|\delta_{overall})/(5|\Pi_{1/T}|T)$ for all $j \in S_{m}$
    \STATE Set $F_j^m = \frac{8|\calS|^2H^4 |\calA|^2\ln(2|\Pi_j^m|/\delta_j^m)}{(\epsilon_{est}^j)^2}$ for all $j \in S_{m}$.
    \STATE Set $n_j^m = 2 \lambda_1\lambda_2 F_j^m $ for all $j \in S_{m}$.
    \STATE Set $N_m = \sum_{j \in S_m} n_j^m$ and $T_m^s = T_{m-1}^s+N_{m-1}$
    \STATE Initialize an independent sub-algorithm for each $j \in S_m$ as $\textsc{EstAll}_j^m = \text{EstAll}(\epsilon_{est}^j,\delta_j^m,F_j^m,\Pi_j^m)$   \label{line: start sub-algorithms}
    \FOR{ $t=T_m^s, T_m^s+1, \ldots, T_m^s + N_m-1$} \label{line: start an epoch}
        \STATE Run $\textsc{EstAll}_j^m.\text{CONTINUE}$ with probability $q_j^t = n_j^m/N_m$
    \ENDFOR    
    \IF{ there exists unfinished $\textsc{EstAll}_j^m$}
        \STATE Set $T_m^s = t + 1$ and repeat the whole process from line 10. ~~~~ $\triangleright$ So each repeat is a sub-epoch.
        % line~\ref{line: start an epoch} \label{line: repeat}
    \ELSE
        \STATE Obtain $\hat{r}_m(\pi)$ for all $\pi$. \label{line: non-repeat}
    \ENDIF
    \STATE Set $\hat{r}_*^m = \max_{\pi \in \Pi_{1/T}} \{\hat{r}_m(\pi) - \frac{1}{16}\hat{\Delta}_\pi^{m-1}\} $ \label{line: construct subset of policies 1}
    \STATE Set $j^m(\pi) = \inf \{ j | 2^{-j} < \max\{2^{-m},\hat{r}_*^m - \hat{r}_\pi^m\}\}$ for all $\pi$, and let $\hat{\Delta}_\pi^m = 2^{-j^m(\pi)} $  \label{line: construct subset of policies 2}
    \STATE Add $\pi$ into $\Pi^{m+1}_{j^m(\pi)}$ for all $\pi$ and set  $S_{m+1} = \bigcup_{\pi} j^m(\pi) $ \label{line: construct subset of policies 3} 
\ENDFOR
\end{algorithmic}
\end{algorithm}

% \paragraph{Notations of epochs:} \kevin{I would embed this section inside algorithm explanation, this is not the first thing people need to understand}We use $E_m$ to denote the $m$-th epoch. Because the epoch will be restarted when there is unfinished \textsc{EstAll} as shown in line 14 and 15, so each $E_m$ can be decomposed into one or several sub-epochs, denoted as $E_m^1, E_m^2, \ldots, E_m^{\Gamma_m}$, each with length $N_m$. In the last sub-epoch, either all the \textsc{EstAll} finished or the whole algorithm ends. 

% \paragraph{Algorithm explanation:} 
Algorithm~\ref{alg:main_1} is based on the multi-arm bandits algorithm \textsc{BARBAR} proposed in \citet{gupta2019better}. In \textsc{BARBAR}, instead of permanently eliminating an arm, the learner will continue pulling each arm with a certain probability defined by its estimated gap. 
Specifically, in an epoch $m$ with length $2^{2m}$, an arm $a$ with an estimated gap $\hat{\Delta}_a^m$ will be pulled roughly $1/(\hat{\Delta}_a^m)^2$ times and suffer roughly $\frac{\text{total corruption in epoch m}}{2^{2m}(\hat{\Delta}_a^m)^2}$ amount of corruptions due to the randomness, so the estimation error of arm $a$ will decrease when the the epoch length doubles, as long as the total amount of corruptions is sublinear. Therefore, close-to-optimal arms that suffered from large corruptions initially can recover and be correctly estimated later, instead of being permanently eliminated at the very beginning.

% line~\ref{line: construct subset of policies 3}

In our algorithm, we regard each policy $\pi$ as an arm and perform the same type of sampling strategy. We denote each repeat from Line 11 to 13 in epoch $m$ as a sub-epoch $E_m^k$ with length $\tilde{\order}(2^{2m})$. Then in any $E_m^k$, each policy $\pi$ with estimated gap $\hat{\Delta}_\pi^m$ will be \textit{simulated} roughly $1/(\hat{\Delta}_\pi^m)^2$ times and will suffer roughly $\frac{C_{E_m^k}^r+C_{E_m^k}^p}{2^{2m}(\hat{\Delta}_\pi^m)^2}$ amount of corruptions. 
While it suffices to rollout each $\pi$ for $\order(1/\epsilon^2)$ episodes to get an $\epsilon$-close estimation, this will result in a $\order(|\Pi_{1/T}|)$ dependence in regret. \textit {In this work, we achieve an $\order(\log(|\Pi_{1/T}|))$ dependence by utilizing the shared information between policies.}

To be specific, at the end of each epoch $m$, we divide the policies into several subsets according to their current estimated policy gap (Line 19 to 21). For example, policies in $\Pi_j^{m+1}$ all have estimated policy gaps close to $2^{-j}$. These subsets will be used for random sampling in the next epoch. And here we use $S_{m+1}$ as a collection of the indices of these subsets.

Now suppose there exists a ``perfect" oracle which guarantees an $\epsilon$-close estimation on each policy uniformly inside some input policy set $\Pi_{est}$, with only $\order(\log(\Pi_{est})/\epsilon^2)$ sample complexity. Then, by calling such an oracle on each subset of polices $\Pi_j^m$, we will able to achieve the simulation goal stated above. Here we propose a reward-free exploration algorithm \textsc{EstAll} as the sub-algorithm, whose performance is close to such a ``perfect" oracle when the the amount of corruptions is relatively small, and still guarantees some sublinear regret otherwise. (See Section~\ref{sec: subAlgo} for details)

\begin{figure}[ht]
\begin{framed}
\textbf{$\textsc{EstAll}_j^m.\text{\rule{0.25cm}{0.15mm}INIT\rule{0.25cm}{0.15mm}}$}\\
\\
-- Start and run an independent sub-algorithm according to the inputs as described in Algorithm~\ref{algo: EstAll} until some policy $\pi$ needs to interact with the environment.\\
--  Suspend this sub-algorithm and set $\pi$ awaiting.
\end{framed}
\end{figure}

\begin{figure}[ht]
\begin{framed}
\textbf{$\textsc{EstAll}_j^m.\text{\rule{0.25cm}{0.15mm}FINISH\rule{0.25cm}{0.15mm}}$}\\
\\
-- Return ``finish" when each $\pi \in \Pi_j^m$ gets an estimation $\hat{r}(\pi)$ as defined in Line 15 in Algorithm~\ref{algo: EstAll}. 
\end{framed}
\end{figure}

\begin{figure}[ht]
\begin{framed}
\textbf{$\textsc{EstAll}_j^m.\text{CONTINUE}$}\\
\\
\textbf{If} $\textsc{EstAll}_j^m$ is suspended\\
    ---- Rollout the awaiting $\pi$ once, which caused the suspension\\
    ---- Continue running the $\textsc{EstAll}_j^m$ as described in Algorithm~\ref{algo: EstAll} until the next $\textsc{rollout}$ is met, which means that there is some policy $\pi'$ that needs to interact with the environment\\
    ---- Suspend the algorithm again and let $\pi'$ be the new awaiting policy\\
\textbf{Else} ~~~~ $\triangleright$ $\textsc{EstAll}_j^m$ has finished\\
    ---- Rollout any $\pi \in \Pi_j^m$ randomly\\
\textbf{end}
\end{framed}
\end{figure}

To be specific, at the beginning of each sub-epoch $E_m^k$, the learner initializes a set of parallel sub-algorithms denoted as $\{\textsc{EstAll}_j^m\}$ corresponding to the constructed subset of policies (Line 10). Here $\delta_j^m$ and $F_j^m$ set in Line 6 and 7 represent a failure probability and a parameter related to the number of roll-outs, given as inputs to $\textsc{EstAll}_j^m$, which is described in Section~\ref{sec: subAlgo} in detail. And $n_j^m$ set in Line 8 is the expected number of times $\textsc{EstAll}_j^m$ will interact with the environment. As described before, such an interaction strategy is carefully randomized according to the estimated gap of policies inside this sub-algorithm (Line 12). Then after roughly $n_j^m = \tilde{\order}(\log(|\Pi_j^m|)/\epsilon_j^2)$ interactions, $\textsc{EstAll}_j^m$ returns one of the following conditions with probability at least $1- \delta_j^m$:
\begin{itemize}
    \item an $\left(\epsilon_j + (C_{E_m^k}^r+C_{E_m^k}^p)\epsilon_j^2\right)$-close estimation on each $\pi$, denoted as $\hat{r}_m(\pi)$, when $\textsc{EstAll}_j^m$ has finished. ( from Theorem~\ref{them (EstAll sec)：value est} )
    \item an unfinished $\textsc{EstAll}_j^m$, which implies that $(C_{E_m^k}^r+C_{E_m^k}^p) \geq \tilde{\Omega}(1/\epsilon_j)$. ( from Theorem~\ref{them (EstAll sec)：sample complexity} )
\end{itemize}
In the first case, we have achieved the desired uniform estimation with $\hat{r}_m(\pi)$ on each policy. (Line 16 and 17) The algorithm will then construct a new subset of policies and go to the next epoch. In the second case, we will repeat the sub-epoch until we successfully obtain uniform estimation on each policy. (Line 14 and 15) Due to the lower bound on $(C_{E_m^k}^r+C_{E_m^k}^p)$, we can show that the regret caused by discarded sub-epochs can be upper bounded in terms of the amount of corruption.

\begin{theorem}
\label{them (BARBAR-RL sec)：regret for BARBAR-RL}
By running this algorithm in the non-cheated setting, with probability at least $1 - \delta_{overall}$, the regret is bounded by 
\begin{align*}
    % &\tilde{\order}\left( |\calS|^{2}|\calA|^{3/2}H\min\{ H^{1/2}, |\calS|^{1/2}|\calA|^{1/2} \} \ln(1/\delta_{overall}) \sqrt{T} \right) \\
    %     & \quad +  \tilde{\order}\left( |\calS|^{4}|\calA|^{3}H^2\min\{ H, |\calS||\calA| \} \ln(1/\delta_{overall})^2 (C^p+C^r) 
    %     +|\calS|^{2}|\calA|^2\ln(1/\delta_{overall}) C^p(C^p+C^r)\right)
    &\tilde{\order}\left( |\calS|^{2}|\calA|^{\frac{3}{2}}H^2\min\{ \sqrt{H}, \sqrt{|\calS||\calA|} \} \ln(1/\delta_{overall}) (\star) \right) \\
        & \quad +  \tilde{\order}\left( 
        |\calS|^2|\calA|^2H^2\ln(1/\delta_{overall}) C^p\right)\\
        & \quad + \tilde{\order}\left( |\calS||\calA|\ln(1/\delta_{overall})C^r\right)\\
        & \quad +  \tilde{\order}\left( \frac{(C^p)^2}{H}
            +\frac{C^pC^r}{H^2}\right)
\end{align*}
where $\tilde{\order}$ hides $\log$ factors on $T,|\calS|,|\calA|,H$, and
\begin{align*}
    \star = \min\{\sqrt{T},\frac{1}{\min_{\pi\in\Pi} \Delta_\pi}\}.
\end{align*}
\end{theorem}
We note that the PolicyGapComplexity, $\frac{1}{\min_{\pi\in\Pi} \Delta_\pi}$, has also been used in some previous work~\citep{jaksch2010near}.
If we let $\Pi$ be all deterministic policies, the PolicyGapComplexity will be close to the GapComplexity defined in \citet{DBLP:conf/nips/SimchowitzJ19} in some non-trivial cases, for example, when all the policies visit a subset of states at step 2 with uniform probability. Otherwise, it can be much larger than the GapComplexity. 
We postpone the discussion on their relation to Appendix~\ref{app: details on gap complexity}. 
\\\\
The dependence on $|\calS|,|\calA|,H$ is not optimal compared to many existing tabular RL results without corruptions, but compared to \citet{lykouris2020corruption}, our result scales better in terms of $H$. 
Most importantly, this is the first result we are aware of in the corrupted setting where the amount of corruptions contributes only additively to the regret bound instead of multiplying $\sqrt{T}$ as in \citet{lykouris2020corruption}.
Conceptually, our result also suggests that corruptions on transition functions have much more influence on the regret than the corruptions on rewards.

Finally, we provide some intuition for why the $\tilde{\order}\left( \frac{(C^p)^2}{H} +\frac{C^pC^r}{H^2}\right)$ terms appear in the bound: Suppose in some epoch there was more than $\order(\sqrt{N_m})$ amount of corruptions, but all the sub-algorithms still happened to finish (e.g., if the adversary changed the transition function in an undetectable way). Furthermore, in the next epoch, the adversary manipulates the corruptions to force the algorithm to restart the sub-algorithms again and again. Under this described scenario, the algorithm is repeatedly using the data from previous corrupted epochs without any chance to correct them, which causes the $(C^p)^2$ and $C^pC^r$ terms. In addition, since $c_t^p$ scales with the horizon $H$ and this regret term depends on the number of times the learner restarts sub-algorithms,  when the total corruption budget $C^p$ is fixed, we will have $H$ in the denominators.

% \yf{Doing a final checking for all the SAH terms+ add more explanations for theorem.}
% \yf{One of the $|\calA|$ might actually be $\min\{|\calA|,H\}$}
% \yf{Do we need a policy gap?}
% \\

\subsection{The Algorithm and the Result for Cheated Adversary}

\paragraph{Algorithm Overview: } In Algorithm~\ref{alg:main_1}, we avoid permanently eliminating an policy. Instead, we use a random policy sampling strategy to ensure that, the corruptions that affected any given policy estimation in the early stages can be corrected for later. However, in the cheated setting, the randomness of policy sampling no longer works because now the adversary decides when to corrupt after seeing the sampled policy. Thus, we propose \textsc{ brute-force policy elimination }, which is based on the traditional policy elimination method that permanently eliminates policies, but with an enlarged confidence range of  $\tilde{\order}(\sqrt{HT})$. Therefore, the best policy will never be eliminated as long as $C^p+C^r \leq \tilde{\order}(\sqrt{HT})$. But such a brute-force method will lead to a regret that scales like $(C^r)^2$ instead of $C^r$. As before, we still need a uniform estimation of each policy with only a $\order(\log |\Pi|/\epsilon^2)$ sample complexity.  Fortunately, the same approach still works, which is, running a set of sub-algorithms in parallel and restarting them when there is an unfinished one. The algorithm and analysis techniques are very similar as in the non-cheated adversary case, and therefore, we postpone the details into Appendix~\ref{sec: app-mainAlgo2}.

\begin{theorem}
\label{them：regret for PolicyElim}
By running this algorithm in the cheated setting, with probability at least $1-\delta_{overall}$, the regret is upper bounded by
\begin{align*}
    &\tilde{\order}\left(|\calS|^2|\calA|^{3/2}H^2 \min\{\sqrt{H},\sqrt{|\calS||\calA|}\} \ln(1/\delta_{overall})\sqrt{T}\right) \\
        &\quad + \tilde{\order}\left( \frac{( C^r)^2}{|\calS||\calA|H^3}
        +|\calS||\calA|H(C^p)^2\right) 
\end{align*}
\end{theorem}
Compared with Theorem~\ref{them (BARBAR-RL sec)：regret for BARBAR-RL}, Theorem~\ref{them：regret for PolicyElim} suffers an additional $\frac{( C^r)^2}{H^3|\calS||\calA|}$ regret and also has additional $H^2|\calS||\calA|$ multiplicative dependence on $(c^p)^2$ terms, to account for the cheated adversary.

% \paragraph{Remark} \kevin{I'm not sure what this means} In Section 2.2 of \cite{bogunovic2020stochastic}, they proved that in order to get $\tilde{\order}(\sqrt{HT})$, the corruption terms can go as low as $\tilde{\Omega}(\frac{C^2}{\log C})$ for the linear bandits. Therefore, we conjecture that $\tilde{\order}((C^r+C^p)^2)$ term is also unavoidable in our setting.

\subsection{Analysis Sketch for Theorem~\ref{them (BARBAR-RL sec)：regret for BARBAR-RL}}

We give a proof sketch for Theorem~\ref{them (BARBAR-RL sec)：regret for BARBAR-RL} here and postpone the details to Appendix~\ref{sec: app-mainAlgo1}.  
\paragraph{Step 1:} Let $\Gamma_m$ denote the number of sub-epochs in epoch $m$. Firstly, appealing to standard concentration inequalities and the random policy sampling strategy, we show that the following events hold with high probability. Note that to aid the exposition, the events defined below are somewhat different than the ones defined in the Appendix.

\begin{align*}
&\calE_{est}:= \\
&\left\{ \forall m,\pi:
\begin{array}{l}
    |\hat{r}^m(\pi) - V_*^\pi|/4\\
    \leq  \lambda_1\lambda_2(C_{E_m^{\Gamma_m}}^r + C_{E_m^{\Gamma_m}}^p)/N_m + \hat{\Delta}_\pi^{m-1}/64
  \end{array}\right\}\\
& \calE_{unfinished}: =  \\
&\left\{ \forall  m, \forall k \in [\Gamma_m-1]:
\begin{array}{l}
    C_{E_m^{k}}^p\\
    \geq \sqrt{\frac{\ln(10T|\Pi_{1/T}|/\delta_{overall})}{16\lambda_1\lambda_2}N_m}
  \end{array}\right\}\\
\end{align*}

Here $\calE_{est}$ suggests that, at the end of epoch $m$, we can have $\tilde{\order}\left(\hat{\Delta}_\pi^{m-1} + (C_{E_m^{\Gamma_m}}^p+C_{E_m^{\Gamma_m}}^r)\epsilon_m^2\right)$-close estimation on every policy. And $\calE_{unfinished}$ suggests that for each unfinished sub-epochs $E_m^k$, its length can always be upper bounded by $\tilde{\order}\left( C_{E_m^{k}}^p\right)$.
% Appealing to standard concentration inequalities and the random policy sampling strategy, we show that these events hold with high probability.

\paragraph{Step 2:} Now we can decompose the regret into 
\begin{align*}
    \Reg
    &\leq \underbrace{\frac{3}{2} \sum_{m=1}^M \sum_{j \in S_m}  \mathring{\Delta}_j^m n_j^{m,\Gamma_m}}_\textsc{non-repeat term} \\
        & \quad + \underbrace{\frac{3}{2} \sum_{m=1}^M \sum_{k=1}^{\Gamma_m -1}\sum_{j \in S_m} \mathring{\Delta}_j^m n_j^{m,k} }_\textsc{repeat term} \\
        & \quad + \order(\text{Low order terms induced by $\epsilon$-net of policies})
\end{align*}
where $\mathring{\Delta}_j^m = \max_{\pi \in \Pi_j^m} \left(\max_{\mathring{\pi} \in \Pi_{1/T}} V_*^{\mathring{\pi}} - V_*^\pi \right)$. 
The non-repeat term represents the sub-epochs where the sub-algorithms complete and estimate all the policy values successfully. Given $\calE_{est}$, by using similar techniques as in \citet{gupta2019better}, we have 
$\mathring{\Delta}_j^m \leq \order(\epsilon_j) +
\order\left(\lambda_1\lambda_2\sum_{s=1}^{m-1}\frac{\left(HC_{E_s^{\Gamma_s}}^p+C_{E_s^{\Gamma_s}}^r\right)}{16^{m-s-1}N_s} \right)$, 
where the second term is a discounted corruption rate. It matches our intuition that the influence from early corrupted estimations will decay as we doubling the epoch.
Thus we can bound the non-repeat term by $\tilde{\order}(\sqrt{T} + C^r+C^p)$. The repeat term represents the regret from sub-epochs when the sub-algorithms restart. Fortunately, according to $\calE_{unfinished}$, this only occurs when the corruption level is beyond some threshold. In this case, intuitively, discarding the data collected in the sub-epoch won't hurt too much since the estimation itself is not accurate. Thus the repeat term can by upper bounded by $\tilde{\order}(C^p(C^r+C^p))$.

\section{The Sub-algorithm and the Results}
\label{sec: subAlgo}

In this section, we give a detailed description for a reward-free exploration algorithm \textsc{EstAll}. As stated in the previous section, we use this algorithm as a black-box sub-algorithm and any improvements in this sub-algorithm would improve the overall regret bounds as well. In a sub-epoch $E_m^k$, we run a set of independent copies in parallel, each denoted as $\textsc{EstAll}_j^m$. As described in $\textsc{EstAll}_j^m.\textsc{continue}$, for each copy $\textsc{EstAll}_j^m$, we will run it offline until some policy needs to interact with the environment. In this case, we will suspend the algorithm and make the policy awaiting hold until the next $\textsc{EstAll}_j^m.\textsc{Continue}$ has been called. Then we will again continue running $\textsc{EstAll}_j^m$ offline and repeat the process above until finished. 

\subsection{Algorithms}
\begin{algorithm}[] 
\caption{ESTALL}
\begin{algorithmic}[1] 
\label{algo: EstAll}
\STATE \textbf{Input:} target estimation error $\epsilon_{est}$, confidence parameter $\delta_{est}$, number of simulate trajectories $F_{est}\geq \frac{8|\calS|^2H^4 |\calA|^2\log(2|\Pi_{est}|/\delta_{est})}{\epsilon_{est}^2}$ and policy set $\Pi_{est}$.
\STATE Set $\tau = 6$, which is a parameter related to \textsc{Rollout}
\STATE Initialize empty buffers $\calD_{s,a}$ for all $(s,a) \in \calS \times \calA$ and let $\calD = \{\calD_{s,a}\}_{(s,a) \in \calS\times\calA}$. 
\STATE Initialize an empty exploration policy set $\Pi_\calD$
\FOR{$\pi \in \Pi$}
    \STATE $\{z_i^\pi\}_{i \in [F]} \leftarrow \textsc{simulate}(\pi,\calD,F_{est})$
    \IF{ $\exists(s,a), \sum_{i=1}^{F_{est}} \one[ z_i^\pi \text{is \textit{Fail} at } (s,a)] \geq \frac{\tau\epsilon_{est}}{|\calS||\calA|H}F_{est}$ } 
    % \label{line (EstAll sec): add pi}
        \STATE $\{z_i^\pi\}_{i \in [F_{est}]}, \calD \leftarrow \textsc{rollout}(\pi,\tau,\calD,F_{est})$
        \STATE $\Pi_\calD \leftarrow \Pi_\calD \bigcup \{\pi\}$ \\
        $\triangleright$ Note that $\Pi_\calD$ is not used in actual algorithm implement, but just for analysis convenience
    \ENDIF
\ENDFOR
\FOR{ each trajectory $z = (s_1,a_1,r_1),(s_2,a_2,r_2),\ldots  $ in $\{z_i^\pi \}_{(i,\pi) \in |F_{est}|\times \Pi_{est}}$} \label{line: update exploration policy set}
    \STATE Calculate 
    \begin{align*}
        r(z) = 
        \Bigg\{
        \begin{array}{r@{}l}
        0 \quad \quad & z \text{ is Fail} \\
        \sum_{h=1}^H r_h \quad \quad& \text{ otherwise}
        \end{array}
    \end{align*}
\ENDFOR
\STATE Calculate $\hat{r}(\pi) = \frac{1}{F_{est}} \sum_{i=1}^{F_{est}} r(z_i^\pi)$ for all $\pi \in \Pi$
\STATE \textbf{return} $\{\hat{r}(\pi)\}_{\pi \in \Pi}$ 
\end{algorithmic}
\end{algorithm}

This algorithm follows the same idea as one in \citet{wang2020long}. That is, we adaptively build an exploration policy set $\Pi_\calD$ and collect samples by only implementing the policies inside $\Pi_\calD$, as shown in \textsc{rollout} (Algorithm~\ref{algo: ROLLOUT}). Then we are able to evaluate many policies simultaneously on the collected data, as shown in \textsc{simulate} (Algorithm~\ref{algo: SIMULATE}). The original version in \citet{wang2020long}, however, requires $\order(poly(|\calS||\calA|H)\log(\Pi)/\epsilon_{est}^3)$ to get a uniform $\epsilon_{est}$-close estimation on each policy values. This is because the original algorithm allocates $\order(poly(H)\log(\Pi)/\epsilon_{est}^2)$  independent sub-algorithms called \textsc{simone}, each with sample complexity $\order(poly(|\calS||\calA|H)/\epsilon_{est})$, and all the data collected in each \textsc{simone} will only be used to simulate one corresponding trajectory of any $\pi$. 

 We improve this algorithm in terms of $\epsilon_{est}$ by the fact that, due to the properties of an MDP, data collected in the one trajectory can be used to simulate different independent trajectories of any $\pi$. Therefore, instead of updating exploration policy set $\Pi_\calD$ based on the failure number on a whole trajectory, we do updates based on the failure number on each state-action pairs. (Line 7 in Algorithm~\ref{algo: EstAll}) Then we show that the size of $\Pi_\calD$ is at most $\tilde{\order}(poly(|\calS||\calA|))$ and each $\pi \in \Pi_\calD$ will interact with environment $\tilde{\order}\left(poly(|\calS||\calA|H)\log(1/\delta_{est})/\epsilon_{est}^2\right)$ times. 
 
 Here $F_{est}$ is the number of trajectories we at least need to simulate each $\pi \in \Pi$ in order to get a desired estimation. Therefore, we need to rollout each $\pi \in \Pi_\calD$ at least $F_{est}$ times. However, while this number is sufficient for simulating $\pi \in \Pi_\calD$ enough times, it does not account for the fact that other policies in $\Pi_\calD$ may need additional data to simulate on. As a consequence we need to repeat the $F_{est}$ rollouts $\tau$ times to ensure we have enough data ($\tau=6$ suffices).

\iffalse
\begin{remark}
Note that if we normalize the rewards for whole trajectory to be $1$, the linear dependence on $H$ in the algorithm in \cite{wang2020long} can be eliminated, which is the target of their work. As a trade-off, they have sub-optimal $\epsilon_{est}$ dependence. In our algorithm, such linear dependence on $H$ exists even after we normalizing the rewards, which is acceptable since our primary focus is on optimizing $\epsilon_{est}$. Whether this linear dependence can be eliminated remains open. \kevin{If we need space, this remark could be cut}
\end{remark}
\fi

\begin{algorithm}[] 
\caption{SIMULATE($\pi,\calD,F$)}
\begin{algorithmic}[1] 
\label{algo: SIMULATE}
\FOR{ $(s,a) \in \calS \times \calA$}
    \STATE Mark all elements in $\calD_{s,a}$ as unused,
\ENDFOR
\FOR{$h \in [H]$}
    \FOR{simulated trajectory $i \in [F]$}
        \IF{all elements in $\calD_{S_h,\pi(s_h)}$ are marked as used}
            \STATE Mark \textit{Fail} at $s_h$ for $i$-th trajectory simulation of $\pi$, denote as $Fail(s_h,\pi_h(s_h),i)$
        \ELSE
            \STATE Set ($s_{h+1}^i,r_{h}^i$) to be the first unused element in $\calD_{S_h,\pi_h(s_h)}$ and mark it as used
        \ENDIF
    \ENDFOR
\ENDFOR
\STATE \textbf{return} \\
$(s_1^i,\pi(s_1)^i,r_1^i),(s_2^i,\pi(s_2)^i,r_2^i), \ldots,(s_H^i,\pi(s_H)^i,r_H^i)$  \\
or\\
$(s_1^i,\pi(s_1)^i,r_1^i),(s_2^i,\pi(s_2)^i,r_2^i), \ldots,$\\
$(Fail(s_h,\pi(s_h),i))$, \\
for all simulated trajectory $i \in [F]$.
\end{algorithmic}
\end{algorithm}

\begin{algorithm}[] 
% \label{algo: ROLLOUT}
\caption{ROLLOUT($\pi$,$\tau$,$\calD$,$F$)}
\begin{algorithmic}[1] 
\label{algo: ROLLOUT}
% \label{alg:EstAll}
\FOR{ $j \in [F \tau]$}
    \STATE Sample the $j$-th trajectory for $\pi$ and collect $H$ samples denoted     as $z_i^\pi = (s_1,a_1,r_1),(s_2,s_2,r_2),\ldots,(s_H,a_H,r_H)$.
    \FOR{$h \in [H]$}
        \STATE Update $\calD_{s_h,a_h} \leftarrow \calD_{s_h,a_h} \cup \{(s_{h+1},r_{h})\}$
    \ENDFOR
\ENDFOR
\STATE \textbf{return} updated $\calD$ and the uniformly chosen $F$ trajectories $\{z_j^\pi\}_{j \in [F]}\}$.
\end{algorithmic}
\end{algorithm}

\subsection{Results and Sketch Analysis}
\begin{theorem}[Sample complexity]
\label{them (EstAll sec)：sample complexity}
     Suppose $F_{est} \geq \frac{8|\calS|^2H^4 |\calA|^2\log(2|\Pi_{est}|/\delta_{est})}{\epsilon_{est}^2}$ and $\tau \geq 6$.
     If the $C_{est}^p \leq \frac{\epsilon_{est}F_{est}}{2|\calS||\calA|H^2}$, then with probability at least $1-\delta_{est}$, the number of (non-simulated) roll-outs in the environment is at most 
    %  Under the corruption assumption $C_{est}^p \leq \frac{\epsilon_{est}F}{2|\calS||\calA|H^2}$, with probability at least $1- \delta_{est}$, the algorithm at most interacts with environment for ,
     \begin{align*}
         |\calS||\calA|F_{est}\tau log(H|\calS||\calA|/\epsilon_{est})
     \end{align*}
     times.
     This also implies that if the algorithm interacts more than the above number of times, then with probability at least $1- \delta_{est}$, $C_{est}^p > \frac{\epsilon_{est}F_{est}}{2|\calS||\calA|H^2}$.
\end{theorem}
\paragraph{Proof Sketch: } Here we provide a proof sketch for the non-corrupted setting and postpone the details including how to deal with $C_{est}^p \leq \frac{\epsilon_{est}F_{est}}{2|\calS||\calA|H^2}$ into Appendix~\ref{sec: app-EstAll}. Notice that, every time the condition in Line 7 in Algorithm~\ref{algo: EstAll} is satisfied, we will add the corresponding $\pi$ into the exploration set $\Pi_\calD$ and rollout $\pi$ in the environment $F_{est}\tau$ times. So the key is to show that, without the presence of corruptions, the number of times the condition in Line 7 in Algorithm~\ref{algo: EstAll} has been satisfied scales like $\order(\log|\Pi_{est}|)$ and not $\order(|\Pi_{est}|)$. 

Define $f^\pi(s,a)$ as the random variable describing the total number of times a single trajectory induced by $\pi$ visits $(s,a)$ under the MDP $\calM^*$. 
If $\sum_{i=1}^{F_{est}} \one[ z_i^\pi \text{is \textit{Fail} at } (s,a)] \geq \frac{\tau\epsilon_{est}}{|\calS||\calA|H}F_{est}$ for some fixed $(s,a)$ and $\pi$, then there are only two cases. In case 1, $|\calD_{s,a}| = 0$ and $\E[f^{\pi}(s,a)] \geq \Omega\left(\frac{\epsilon_{est}}{|\calS||\calA|H}F_{est}\right)$. So calling $\textsc{rollout}(\pi,\tau,\calD,F_{est})$ will make  $|\calD_{s,a}|$ increase to at least $o\left(\frac{\epsilon_{est}}{|\calS||\calA|H}F_{est}\right)$ with high probability. In case 2, $|\calD_{s,a}|$ is roughly smaller than $2\E[f^{\pi}(s,a)]F_{est}$. So calling $\textsc{rollout}(\pi,\tau,\calD,F_{est})$ will make  $|\calD_{s,a}|$ double with high probability. (Notice here we say ``roughly" because in the actual proof, we consider some lower bound of $|\calD_{s,a}|$ instead of $|\calD_{s,a}|$ directly.) 
Thus, $|\calD_{s,a}|$ starting in the worst case at about $\frac{\epsilon_{est}}{|\calS||\calA|H}F_{est}$ will eventually double until it reaches $H F_{est}$, at which time the simulation will never fail. 
Therefore, the total number of polices added into $|\Pi_\calD|$ due to the failure at $(s,a)$ is about $\log_2( (H F_{est}) / (\frac{\epsilon_{est}}{|\calS||\calA|H}F_{est}) = \log_2(H^2|\calS||\calA|/\epsilon_{est} ) )$. Noting that there are $|\calS||\calA|$ number of state-action pairs, and $F_{est}\tau$ trajectories are taken per added policy, we conclude the proof. 
\newline
\begin{theorem}[Estimation correctness]
\label{them (EstAll sec)：value est}
     Suppose $F_{est} \geq \frac{8|\calS|^2H^4 |\calA|^2\log(2|\Pi_{est}|/\delta_{est})}{\epsilon_{est}^2}$ and $\tau \geq 6$. Then for all $\pi \in \Pi$, with probability at least $1 - \delta_{est}$,
     \begin{align*}
          \big| \hat{r}(\pi) - V^{\pi}(s_1) \big| \leq (1+\tau)\epsilon_{est} + (HC_{est}^p+C_{est}^r)/F_{est}
     \end{align*}
\end{theorem}
\paragraph{Proof Sketch: } We provide a proof sketch here and postpone the details until Appendix~\ref{sec: app-EstAll}.
By definition, $\hat{r}(\pi) = \frac{1}{F_{est}}\sum_{i=1}^{F_{est}} r(z_i^\pi)$ and  $\{r(z_i^\pi) \}_{i=1}^{F_{est}}$ is a sequence of independent random variables. We denote their expected values $\E[r(z_i^\pi)]$ as $\{V^\pi_i \}_{i=1}^{F_{est}}$. Here $V_i^\pi$ is not a true value function but an ``average value function'' whose rewards and transition functions are the average of rewards and transition functions generated by the MDPs under different times (so some are corrupted). 

Now, for those $\pi \in \Pi_\calD$, we can use Hoeffding's inequality to directly bound $\big| \hat{r}(\pi) - \frac{1}{F_{est}}\sum_{i=1}^{F_{est}} V^\pi_i \big|$.
For those $\pi \notin \Pi_\calD$, if none of them are failed, we can again use Hoeffding's inequality to directly bound $\big| \hat{r}(\pi) - \frac{1}{F_{est}}\sum_{i=1}^{F_{est}} V^\pi_i \big|$. Otherwise, because the policy \textit{fails} at most $\epsilon_{est} \tau F/H|\calS||\calA|$ times at each $(s,a)$ according to Line 7 in Algorithm~\ref{algo: EstAll}, there will be at most $\tau \epsilon_{est} F_{est}/H$ trajectories with \textit{fails} when  computing $\hat{r}(\pi)$. Thus, $\hat{r}(\pi)$ is changed at most by $\tau \epsilon_{est}$ from the no-failure case and we get the following,
\begin{align*}
    \prob\left[ \big| \hat{r}(\pi) - \frac{\sum_{i=1}^{F_{est}}V^\pi_i}{F_{est}}  \big| \geq (1+\tau)\epsilon_{est}\right]
    \leq \delta_{est}/2|\Pi_{est}|
\end{align*}
Now we can decompose out target result into,
\begin{align*}
    \big| \hat{r}(\pi) - V^{\pi} \big|
    \leq  \big| \hat{r}(\pi) -\frac{\sum_{i=1}^{F_{est}} V^\pi_i}{F_{est}} \big| + \big|\frac{\sum_{i=1}^{F_{est}} V^\pi_i}{F_{est}} - V^\pi \big|
\end{align*}
The first term can be upper bounded by the previous results. The second term can be upper bounded by the total corruptions. 
Finally, by taking union bound over all policy in $\Pi_{est}$, we get our target result.

% Theorem~\ref{them (main paper): value estimation}
\section{Discussion}
% We obtained the first regret bound in the corrupted setting where the amount of corruptions contributes only additively to the regret bound instead of multiplicatively.
Since our bound in the non-cheated setting scales like $\order((C^p)^2)$, one natural open question is to obtain an $\order(C^p)$ regret bound. 
Second, the computational complexity of our algorithms scale with $|\Pi|$ due to the reward-free exploration sub-algorithm we use. Thus, finding an efficient algorithm is also an interesting problem. 
% One possible direction is to adaptively maintain a distribution over all polices over time, such as the techniques in \citet{pmlr-v32-agarwalb14}. 
Finally, our algorithm is not instance-dependent, so whether we can achieve some regret of the form $\tilde{\order}\left(\text{GapComplexity} + (C^p+1)(C^p+C^r)\right)$ also remains open.

\iffalse
% Acknowledgements should only appear in the accepted version.
\section*{Acknowledgements}

\textbf{Do not} include acknowledgements in the initial version of
the paper submitted for blind review.

If a paper is accepted, the final camera-ready version can (and
probably should) include acknowledgements. In this case, please
place such acknowledgements in an unnumbered section at the
end of the paper. Typically, this will include thanks to reviewers
who gave useful comments, to colleagues who contributed to the ideas,
and to funding agencies and corporate sponsors that provided financial
support.
\fi

% In the unusual situation where you want a paper to appear in the
% references without citing it in the main text, use \nocite
% \nocite{langley00}

\bibliography{ref}
\bibliographystyle{icml2021}

%%%%%%%%%%%%%%%%%%%%%%%%%%%%%%%%%%%%%%%%%%%%%%%%%%%%%%%%%%%%%%%%%%%%%%%%%%%%%%%
%%%%%%%%%%%%%%%%%%%%%%%%%%%%%%%%%%%%%%%%%%%%%%%%%%%%%%%%%%%%%%%%%%%%%%%%%%%%%%%
% DELETE THIS PART. DO NOT PLACE CONTENT AFTER THE REFERENCES!
%%%%%%%%%%%%%%%%%%%%%%%%%%%%%%%%%%%%%%%%%%%%%%%%%%%%%%%%%%%%%%%%%%%%%%%%%%%%%%%
%%%%%%%%%%%%%%%%%%%%%%%%%%%%%%%%%%%%%%%%%%%%%%%%%%%%%%%%%%%%%%%%%%%%%%%%%%%%%%%

\onecolumn
\appendix

% \textcolor{red}{One correction: We notice that in the main paper, we sometimes write $\log(2|\Pi^m|/\delta^m)$, which should actually be $\ln(2|\Pi^m|/\delta^m)$.}

\section{Organization of appendix}
In Appendix~\ref{sec: app-mainAlgo1}, we give detailed proofs for Theorem~\ref{them (BARBAR-RL sec)：regret for BARBAR-RL}, which is the result for the non-cheated setting. In Appendix~\ref{sec: app-mainAlgo2}, we describe the algorithm omitted in the main paper for the cheated setting as well as its proofs. Then in Appendix~\ref{sec: app-EstAll}, we give detailed proofs for Theorem~\ref{them (EstAll sec)：sample complexity} and \ref{them (EstAll sec)：value est}, which are the results for the reward-free exploration sub-algorithm. Finally, in Appendix~\ref{sec: app-RFalgo}, we give a justification on why efficient reward-free exploration methods proposed in \citet{kaufmann2020adaptive} and \citet{menard2020fast} are difficult to be used as sub-algorithms here.

\section{Regret Analysis for Theorem~\ref{them (BARBAR-RL sec)：regret for BARBAR-RL} (the non-cheated case)}
\label{sec: app-mainAlgo1}

\subsection{Notations}

We use $E_m$ to denote the $m$-th epoch. Because the epoch will be restarted when there is an unfinished \textsc{EstAll} as shown in line 14 and 15, each $E_m$ can be decomposed into one or more sub-epochs, denoted as $E_m^1, E_m^2, \ldots, E_m^{\Gamma_m}$, each with length $N_m$. In the last sub-epoch, either all the \textsc{EstAll} are finished or the whole algorithm ends. 

For convenience, we also define the following notations
\begin{itemize}
    % \item Denote $C_{E_m}^{r(p)}$ as $C_m^{r(p)}$, denote $C_{E_m^k}^{r(p)}$ as $C_{m,k}^{r(p)}$. Also when we use $C$ without superscript, it simply means $C^p+C^r$. \yf{change $r(p)$}
    \item $\mathring{\pi} = \argmax_{\pi \in \Pi_{1/T}} V_*^\pi$, $\mathring{V} = V_*^{\mathring{\pi}}$ and $\mathring{\Delta}_{\pi} = \mathring{V} - V_*^\pi$,
    \item $\pi_*^m =  \argmax_{\pi\in\Pi_{1/T}} \{\hat{r}_m(\pi) - \frac{1}{16}\hat{\Delta}_\pi^{m-1}\}$ 
    \item $\tilde{n}_j^{m,k}$ be the real number of times that policy set $\Pi_j^m$ interacting with environment inside $E_m^k$
    \item $\rho_m = \sum_{s=1}^{m}\frac{8\lambda_1\lambda_2(HC_s^p+C_s^r)}{16^{m-s}N_s}$
    \item  $\mathring{\Delta}_j^m = \max_{\pi \in \Pi_j^m} \mathring{\Delta}_\pi$.
\end{itemize}

\subsection{High Probability Events}
We define the following events and show that these events occur with high probability. 
\begin{definition}
    Define an event $\calE_{overall}$ which implies that the actual length of all sub-algorithms is closed to their scheduled time 
    \begin{align}
    \label{event (BARBAR=RL): overall}
        \calE_{overall}: = \left\{\forall m, \forall k \in [\Gamma_m], \forall j \in [S_m]: \tilde{n}_j^{m,k} \in [\frac{1}{2}n_j^m, \frac{3}{2} n_j^m ] \right\}
    \end{align}
\end{definition}

\begin{definition}
    Define an event $\calE_{est}$, which implies that, for all the completed sub-epochs, we can estimated all the policy uniformly at the end of epoch 
    \begin{align*}
    \calE_{est}: = \left\{\forall m,\pi: |\hat{r}_m(\pi) - V_*^\pi| \leq 2\lambda_1\lambda_2\frac{2(HC_{m,k}^p+C_{m,k}^r)}{N_m} + \frac{1}{16}\hat{\Delta}_\pi^{m-1} \right\}
\end{align*}
\end{definition}

\begin{definition}
    Define an event $\calE_{unfinished}$, which implies that, for all sub-epochs with unfinished sub-algorithm, we always have large corruption as long as $\calE_{overall}$ holds,
    \begin{align*}
        \calE_{unfinished}: = \left\{\forall m, \forall k \in [\Gamma_m]: C_{m,k}^p 
        \geq \frac{1}{4}\sqrt{\frac{\ln(10T|\Pi_{1/T}|/\delta_{overall})}{\lambda_1\lambda_2}N_m} \right\} \text{ and } \calE_{overall} 
    \end{align*}
\end{definition}
Now we are going to prove that $\prob[\calE_{overall} \cap \calE_{est} \cap \calE_{unfinished}] \geq 1-\delta_{overall}$. We first show that with high probability, $\calE_{overall}$ holds,
\begin{lemma}[High Probability for $\calE_{overall}$]
\label{lem (BARBAR-RL sec): high prob for calE_{overall} }
        $\prob\left[ \calE_{overall} \right] \geq 1- \delta_{overall}/4 $
\end{lemma}
\begin{proof}
    For any fixed $E_m^k$ and $\Pi_j^m$, we use a Chernoff-Hoeffding bound on the r.v. $\tilde{n}_j^{m,k}$. The expected value is $\E [\tilde{n}_j^m] = n_j^m \geq \lambda_2 = 12 \log(8T/\delta_{overall}$), so 
    \begin{align*}
        \prob\left[  |\tilde{n}_j^{m,k} -  n_j^m | \geq \frac{1}{2}n_j^m \right] 
        \leq 2 \exp \left( - (\frac{1}{4}n_j^m)/3 \right)
        \leq \delta_{overall}/4T\log(T)
    \end{align*}
    Because of the possible failure of a sub-algorithm, there will be at most $T$ sub-epochs and $\log(T)$ sub-policy sets. So by taking the union bound over all the sub-epochs and sub-policy sets, we get the target result
\end{proof}
Next, we are going to show with high probability we have $\calE_{overall} \cap \calE_{est}$. But before we actually prove those, we will first prove the following lemma that gives an estimation on the total amount of corruptions that will be included in each sub-algorithm.

\begin{lemma}
\label{lem (BARBAR-RL sec): total amount of corruption for EstAll}
% \yf{$c_t^rp$ is in range $[0,H]$ so here something need  to be fixed}
    For any fixed sub-epoch $E_m^k$ and any fixed  $\Pi_j^m$, we have 
    \begin{align*}
        \prob \left[ \sum_{t \in E_m^k} c_t^p \one\{\pi_t \in \Pi_j^m\} \geq \frac{2n_j^m}{N_m} C_{m,k}^p + H\ln 4/\delta \text{ and }\sum_{t \in E_m^k} c_t^r \one\{\pi_t \in \Pi_j^m\} \geq \frac{2n_j^m}{N_m} C_{m,k}^r + H\ln 4/\delta\right] \leq \frac{\delta}{4}
    \end{align*}
\end{lemma}
\begin{proof}
    It follows a very similar proof of Eqn.3 in \cite{gupta2019better}. Let $Y_j^t = \one\{\pi_t \in \Pi_j^m\}$ and $B_j^m = \sum_{t \in E_m}Y_j^t c_t^{rp}$. Notice that $Y_j^t$ is an independent Bernoulli variable with mean $q_j^t$. Consider the sequence of r.v.s $X_1,\ldots,X_{N_m}$ defined by $X_{t-T_m^s+1} = (Y_j^t - q_j^t)c_t^{rp}$ for $t \in E_m$. Then it is a martingale difference sequence with predictable quadratic variation $Var = q_j^m \sum_{t \in E_m}c_t^{rp}$. Then by applying the freedman inequality we get that, with probability at least $1 - \delta$,
    \begin{align*}
        B_j^m \leq q_j^m \sum_{t \in E_m^k}c_t^{rp} + (Var/H + H\ln 4/\delta) \leq  2q_j^m \sum_{t \in E_m}c_t^{rp}  + H\ln 4/\delta
    \end{align*}
    By replacing $q_j^m = n_j^m/N_m$ and $\sum_{t \in E_m^k}c_t^{rp}  \leq C_{m,k}^{rp}$ into that, we have
    $ B_j^m \leq \frac{2n_j^m}{N_m} C_{m,k}^{rp} + H\ln 4/\delta $
\end{proof}
We now continue proving our claim:
\begin{lemma}[High Probability for $\calE_{est}$]
\label{lem (BARBAR-RL sec): high prob for calE_{est} }
     $\prob\left[ \calE_{est} \right] \geq 1- \delta_{overall}/4 $ 
    %  \yf{choose $\mathring{\Delta}_j^m = \delta_{overall}/(5\log(T)^2)$, or $\sum_j \mathring{\Delta}_j^m = \delta_{overall}/(5\log(T))$, the latter one will be smaller}
\end{lemma}
\begin{proof}
For any fixed $m,j$, suppose the $\textsc{EstAll}_j^m$ is completed. From Lemma~\ref{lem (BARBAR-RL sec): total amount of corruption for EstAll}, we know that, with high probability $1-\delta_j^m/4$, there will be at most $\left(\frac{2n_j^m}{N_m} C_{m,k}^{rp} + H\ln(4/\delta_j^m)\right)$ amount of corruptions included in the sub-algorithm $\textsc{EstAll}_j^m$. Then by Theorem~\ref{them (EstAll sec)：value est}
% \ref{them (EstAll sec): value estimation (key)  }
, we have that, with probability as least $1 - \delta_j^m$, for all $\pi \in \Pi_j^m$
\begin{align*}
    \big| \hat{r}_m(\pi) - V_*^\pi \big| 
    & \leq 7\epsilon_{est}^j +   \frac{n_j^m}{F_j^m}\left(\frac{2(HC_{m,k}^p+C_{m,k}^r)}{N_m} \right )  + \frac{H\ln(4/\delta_j^m)}{F_j^m}  \\
    &\leq 7\epsilon_{est}^j + 2\lambda_1\lambda_2 \left(\frac{2(HC_{m,k}^p+C_{m,k}^r)}{N_m} \right )
    + \epsilon_{est}^j \\
    &\leq \frac{1}{16}\epsilon_j + 2\lambda_1\lambda_2 \left(\frac{2(HC_{m,k}^p+C_{m,k}^r)}{N_m} \right ) 
\end{align*}
Now by taking the union bound over at most $\log T$ epochs and at most $\log T$ sub-algorithms for each epoch, as well as replacing the value of $\mathring{\Delta}_j^m$, we have that, with probability at least $1 - \delta_{overall}/4$, for all $m,j$ and all $\pi \in \Pi_j^m$
\begin{align*}
    |\hat{r}_m(\pi) - V_*^\pi| \leq \epsilon_{j}/16 + 2\lambda_1\lambda_2 \frac{2(HC_{m,k}^p+C_{m,k}^r)}{N_m}
\end{align*}
By the definition of $\hat{\Delta}_\pi^m$ and $\Pi_j^m$, this can also be written as, for all $m$ and all $\pi \in \Pi$, with probability at least $1 - \delta_{overall}/4$,
\begin{align*}
    |\hat{r}_m(\pi) - V_*^\pi| \leq \hat{\Delta}_\pi^m/16 + 2\lambda_1\lambda_2\frac{2(HC_{m,k}^p+C_{m,k}^r)}{N_m}
\end{align*}
\end{proof}

\begin{lemma}[High Probability for $\calE_{unfinished}$]
\label{lem: lower bound for corruption on failed epoch}
    $\prob\left[ \calE_{unfinished} \right] \geq 1- \delta_{overall}/4 $ 
\end{lemma}
\begin{proof}
    Given $\calE_{overall}$, all the $\textsc{EstAll}_j^{m,k}$ will have more than $\tilde{n}_j^m \geq \lambda_1 F_j^m \geq 6|\calS||\calA|F_j^m \log(H|\calS||\calA|)$ number of interactions with the environment.
    Then by Theorem~\ref{them (EstAll sec)：sample complexity}
    % \ref{them (EstAll sec): maximun interaction number of EstAll}
    , we know that since $\textsc{EstAll}_j^{m,k}$ is unfinished, then with probability at least $1 - \delta_j^m$, we will have more than $\frac{\epsilon_{est}^j}{2|\calS||\calA|H^2}F_j^m$ amount of corruptions being included in any fixed $\textsc{EstAll}_j^{m,k}$. 
    \\\\
    Next by Lemma~\ref{lem (BARBAR-RL sec): total amount of corruption for EstAll}, we know that with probability at least $1-\delta_j^m/4$,
    \begin{align*}
        \frac{2n_j^m}{N_m} C_{m,k}^p + H\ln (4/\delta_j^m) \geq \frac{\epsilon_{est}^j}{2|\calS||\calA|H^2}F_j^m
    \end{align*}
    By replacing the values of $2n_j^m,F_j^m$ and $\epsilon_{est}^j$, we have for any fixed $\textsc{EstAll}_j^{m,k}$,
    \begin{align*}
        2\lambda_1\lambda_2 \left( \frac{2C_{m,k}^p}{N_m}\right) 
        \geq \epsilon_{est}^j \left(\frac{1}{2|\calS||\calA|H^2} - \frac{\epsilon_j}{96|\calS||\calA|H^2} \right)
        \geq \frac{1}{4|\calS||\calA|H^2}\epsilon_{est}^j 
        % \yf{127/96}
    \end{align*}
    Rearranging the inequality we get 
    \begin{align*}
        C_{m,k}^p 
        \geq \frac{1}{16|\calS||\calA|H^2}\frac{N_m}{\lambda_1\lambda_2} \epsilon_{est}^j
        \geq \frac{N_m \epsilon_{est}^m}{16|\calS||\calA|H^2\lambda_1\lambda_2}
        \geq \frac{1}{4}\sqrt{\frac{\ln(10T|\Pi_{1/T}|/\delta_{overall})}{\lambda_1\lambda_2}N_m} 
        % \geq \frac{20H}{4\tau_{max}}\sqrt{\frac{\log(4/\delta)}{\lambda}} * \sqrt{N_m}\yf{?} \yf{?}
    \end{align*}
    where the third inequality comes from the fact that $\epsilon_{est}^m \geq 4H^2|\calS||\calA|\sqrt{\frac{\lambda_1\lambda_2\log(10T|\Pi_{1/T}|/\delta_{overall})}{N_m}}$, which is an rearrangement from the inequality in Lemma~\ref{lem (BARBAR-RL sec) : bound on the epoch length}.
    \\\\
    Finally, we know there are at most $T$ number of sub-epochs. So by taking the union bound over all the sub-epochs and over all the sub-policy set $\Pi_j^m$ inside each sub-epoch $E_m^k$, we get the target result.
\end{proof} 

In what follows we assume events $\calE_{overall}.\calE_{est}$ and $\calE_{unfinished}$ hold, since they do so with probability at least $1-\delta_{est}$.

\subsection{Auxiliary Lemmas}

\begin{lemma}
\label{lem (BARBAR-RL sec) : bound on the epoch length}
    The length of $N_m$ of epoch $m$ satisfies
    \begin{align*}
        16*128^2\lambda_1\lambda_2|\calS|^2H^4 |\calA|^2\ln(10T|\Pi_{1/T}|/\delta_{overall})/(\epsilon_m)^2 
        \leq N_m 
        \leq 64*128^2\lambda_1\lambda_2|\calS|^2H^4 |\calA|^2 10T\log(2/\delta_{overall})/(\epsilon_m)^2
    \end{align*}
    Sometimes we will use the following
    \begin{align*}
        16\lambda_1\lambda_2|\calS|^2H^4 |\calA|^2\ln(10T|\Pi_{1/T}|/\delta_{overall})/(\epsilon_{est}^m)^2 
        \leq N_m 
        \leq 64\lambda_1\lambda_2|\calS|^2H^4 |\calA|^2 10T\log(2/\delta_{overall})/(\epsilon_{est}^m)^2
    \end{align*}
    % \yf{Other inequality}
    % \begin{align*}
    %     & \epsilon_{est}^j = \sqrt{6\lambda_1\lambda_2|\calS|^2H^2|\calA|^2\log(2|\Pi_j^m|/\mathring{\Delta}_j^m)/n_j^m}
    % \end{align*}
\end{lemma}
\begin{proof}
    Because $\hat{r}_*^m - \hat{r}_m(\pi_*^m) \leq 0$, so it has $\hat{\Delta}_{\pi_*^m}^m = \epsilon_m.$ This immediately implies the lower bound as
    \begin{align*}
        N_m 
        \geq \min_{j \in S_m} n_j^m
        \geq 16*128^2\lambda_1\lambda_2|\calS|^2H^4 |\calA|^2\ln(10T|\Pi_{1/T}|/\delta_{overall})/(\epsilon_m)^2
    \end{align*}
     We get the upper bound from the fact that 
    \begin{align*}
        N_m 
        = \sum_{j \in S_m} n_j^m
        \leq 64*128^2\lambda_1\lambda_2|\calS|^2H^4 |\calA|^2\ln(10T|\Pi_{1/T}|/\delta_{overall})/(\epsilon_m)^2
    \end{align*}
\end{proof}

\subsection{Lemmas related to completed sub-algorithm}
\label{sec (BARBAR_RL sec): completed subalgo}
% In what follows we assume events $\calE_{est}$ hold, since they do so with probability at least $1-\delta_{est}$.

In the case that all the sub-algorithms are completed, the proof steps are the very similar to the ones in \cite{gupta2019better}. Here we restate and refined related lemmas.

\begin{lemma}[similar to Lemma 5 \cite{gupta2019better}]
\label{lem: error on best policy estimation}
    Suppose that $\calE_{est}$ occurs. Then for all epochs $m$, 
    \begin{align*}
        - 2\lambda_1\lambda_2\frac{2(HC_m^p+C_m^r)}{N_m} -  \frac{2}{16} 
        \hat{\Delta}_{\mathring{\pi}}^{m-1} \leq \hat{r}_*^m - \mathring{V} \leq 2\lambda_1\lambda_2 \frac{2(HC_m^p+C_m^r)}{N_m}.
    \end{align*}
\end{lemma}
\begin{proof}
For the upper bound, by the definition of $\hat{r}_*^m$ and the occurrence of $\calE_{est}$, we have
    \begin{align*}
        \hat{r}_*^m
        &=\hat{r}_m(\pi_*^m) - \frac{1}{16}\hat{\Delta}_{\pi_*^m}^{m-1}\\
        &\leq V_*^{\pi_*^m} + 2\lambda_1\lambda_2H \frac{2(HC_m^p+C_m^r)}{N_m} + \frac{1}{16}\hat{\Delta}_{\pi_*^m}^{m-1}- \frac{1}{16}\hat{\Delta}_{\pi_*^m}^{m-1}\\
        % &\leq \max_{\pi \in \Pi_{1/T}} \{V_*^\pi + 2\lambda_1\lambda_2H \frac{2(HC_m^p+C_m^r)}{N_m} + \frac{1}{16}\hat{\Delta}_{\pi}^{m-1} - \frac{1}{16}\hat{\Delta}_{\pi}^{m-1} \} \\
        &\leq \mathring{V} + 2\lambda_1\lambda_2 \frac{2(HC_m^p+C_m^r)}{N_m} + \frac{1}{16}\hat{\Delta}_{\pi_*^m}^{m-1} - \frac{1}{16}\hat{\Delta}_{\pi_*^m}^{m-1}
        =  \mathring{V} + 2\lambda_1\lambda_2 \frac{2(HC_m^p+C_m^r)}{N_m}.
    \end{align*}
For the lower bound, we have
\begin{align*}
     \hat{r}_*^m
     \geq \hat{r}_m(\mathring{\pi}) - \frac{1}{16} \hat{\Delta}_{\mathring{\pi}}^{m-1}
     \geq \mathring{V} - 2\lambda_1\lambda_2\frac{2(HC_m^p+C_m^r)}{N_m} -  2\frac{1}{16} \hat{\Delta}_{\mathring{\pi}}^{m-1} 
\end{align*}
\end{proof}

\begin{lemma}[similar to Lemma 6 \cite{gupta2019better}]
\label{lem: gap estimation 1}
    Suppose that $\calE_{est}$ occurs. Then for all epoch $m$ and all policies $\pi$
    \begin{align*}
        \hat{\Delta}_\pi^m \leq 2 \left(\mathring{\Delta}_\pi + 2^{-m} + \sum_{s=1}^{m}\frac{8\lambda_1\lambda_2(HC_s^p+C_s^r)}{16^{m-s}N_s} \right)
    \end{align*}
    % \yff{$2^{-m}$ term can be tighter. pending}
\end{lemma}
\begin{proof}
    The proof is by induction on $m$. For $m=1$, the claim is trivially true because $\hat{\Delta}_\pi^1 \leq 2*2^{-1}=1$. Next, suppose that the claim holds for $m-1$. Using Lemma~\ref{lem: error on best policy estimation} and the definition of $\calE_{est}$, we write 
    \begin{align*}
        \hat{r}_*^m - \hat{r}_m(\pi)
        & = (\hat{r}_*^m - \mathring{V}) + (\mathring{V} - V_*^\pi) + (V_*^\pi - \hat{r}_m(\pi))\\
        & \leq 2\lambda_1\lambda_2 \frac{2(HC_m^p+C_m^r)}{N_m} + \mathring{\Delta}_\pi + 2\lambda_1\lambda_2 \frac{2(HC_m^p+C_m^r)}{N_m} + \frac{1}{16}\hat{\Delta}_\pi^{m-1}
    \end{align*}
    Now using the induction hypothesis, we have
    \begin{align*}
        \hat{r}_*^m - \hat{r}_m(\pi)
        & \leq \mathring{\Delta}_\pi + 2\lambda_1\lambda_2 \frac{4(HC_m^p+C_m^r)}{N_m} + \frac{1}{16} \left(2\mathring{\Delta}_\pi + 2*2^{-(m-1)} + \sum_{s=1}^{m-1}\frac{8\lambda_1\lambda_2(HC_m^p+C_m^r)}{16^{m-1-s}N_s} \right)\\
        & \leq 2\mathring{\Delta}_\pi + 2*2^{-m} + \sum_{s=1}^{m}\frac{8\lambda_1\lambda_2(HC_s^p+C_s^r)}{16^{m-s}N_s}
    \end{align*}
    Now by the definition of $\hat{\Delta}_\pi^m$, if $\hat{r}_*^m - \hat{r}_m(\pi) \leq 2^{-m}$, then we directly have $\hat{\Delta}_\pi^m < 2^{-m}$. Otherwise if $\hat{r}_*^m - \hat{r}_m(\pi) > 2^{-m}$, then $\hat{\Delta}_\pi^m < \hat{r}_*^m - \hat{r}_m(\pi)$
\end{proof}

\begin{lemma}[similar to Lemma 7 \cite{gupta2019better}]
\label{lem: gap estimation 2}
    Suppose that $\calE_{est}$ occurs. Then for all epochs $m$ and all policies $\pi$
    \begin{align*}
        \hat{\Delta}_\pi^m \geq \frac{1}{4}\mathring{\Delta}_\pi  - 3\sum_{s=1}^{m}\frac{8\lambda_1\lambda_2(HC_s^p+C_s^r)}{16^{m-s}N_s} - \frac{3}{8}2^{-m}
        : = \frac{1}{4}\mathring{\Delta}_\pi  - 3\rho_m - \frac{3}{8}2^{-m}
    \end{align*}
\end{lemma}
\begin{proof}
    \begin{align*}
        \hat{\Delta}_\pi^m 
        & \geq \frac{1}{2} (\hat{r}_*^m - \hat{r}_m(\pi)) \\
        & \geq \left(\frac{\mathring{V}}{2} - \lambda_1\lambda_2\frac{2(HC_m^p+C_m^r)}{N_m} - \frac{1}{16}\Delta_{\mathring{\pi}}^{m-1} \right) 
            - \left(\frac{V_*^\pi}{2} + \lambda_1\lambda_2\frac{2(HC_m^p+C_m^r)}{N_m} + \frac{1}{32}\mathring{\Delta}_\pi^{m-1} \right) \\
        & = \frac{\mathring{\Delta}_\pi}{2} - \lambda_1\lambda_2\frac{4C_m}{N_m} - \frac{3}{32}\hat{\Delta}_{\mathring{\pi}}^{m-1} \\
        & \geq \frac{\mathring{\Delta}_\pi}{2} - \lambda_1\lambda_2\frac{4C_m}{N_m} - \frac{6}{32} \left(\mathring{\Delta}_\pi + 2^{-(m-1)} + \sum_{s=1}^{m-1}\frac{8\lambda_1\lambda_2(HC_s^p+C_s^r)}{16^{m-s}N_s} \right)\\
        & \geq \frac{1}{4}\mathring{\Delta}_\pi  - 3\underbrace{\sum_{s=1}^{m}\frac{8\lambda_1\lambda_2(HC_s^p+C_s^r)}{16^{m-s}N_s}}_{\rho_m} - \frac{3}{8}2^{-m}
    \end{align*}
    The first inequality is by the definition of $\hat{\Delta}_\pi^m$. The first term of the second inequality comes from Lemma~\ref{lem: error on best policy estimation} and the second term of the second inequality comes from the occurrence of $\calE_{est}$. And the third inequality comes from Lemma~\ref{lem: gap estimation 1}.
\end{proof}

\begin{cor}
\label{cor: gap estimation 2}
Suppose that $\calE_{est}$ occurs. Then for all epoch $m$ and all policies $\pi$.
\begin{align*}
    \epsilon_j \geq \frac{1}{4}\mathring{\Delta}_j^m  - 3\rho_{m-1} - \frac{3}{8}2^{-(m-1)}
\end{align*}
\end{cor}
\begin{proof}
    The above lemma~\ref{lem: gap estimation 2} holds for all $\pi \in \Pi_j^m$ including the one leads to $\Delta_j^m$. Furthermore, we have $\epsilon_j =  \hat{\Delta}_\pi^{m-1} $. Therefore, we get the target result.
\end{proof}

\subsection{Lemmas related to unfinished sub-algorithms}
\label{sec (BARBAR_RL sec): unfinished subalgo}
Now we will show that, if the sub-algorithm is unfinished, then the number of repeated sub-epochs can be upper bounded in terms of corruption. 

\begin{lemma}
% \yf{can this lemma leads to a better bound ?}
\label{lem (BARBAR_RL sec): upper bound for repeated epoch numbers}
    If $\calE_{unfinished}$ occurs, then we have
    \begin{align*}
        \Gamma_m - 1 
        \leq C_m^p \epsilon_{m}/(H^2|\calS||\calA|\ln(10T|\Pi_{1/T}|/\delta_{overall})
        \leq C_m^p/(H^2|\calS||\calA|\ln(10T|\Pi_{1/T}|/\delta_{overall})
    \end{align*}
\end{lemma}
\begin{proof}
Condition on $\calE_{unfinished}$, we have 
% \kevin{Inconsistency between $\log$ and $\ln$}
\begin{align*}
     N_m 
     &\leq \frac{16\lambda_1\lambda_2}{\ln(10T|\Pi_{1/T}|/\delta_{overall})} \min_{k \in [\Gamma_m-1]}{(C_{m,k}^p)^2} \\
     &\leq \frac{16\lambda_1\lambda_2}{\ln(10T|\Pi_{1/T}|/\delta_{overall})} (\frac{C_m^p - C_{m,\Gamma_m}^p}{\Gamma_m-1})^2\\
     &\leq \frac{16\lambda_1\lambda_2}{\ln(10T|\Pi_{1/T}|/\delta_{overall})} (\frac{C_m^p}{\Gamma_m-1})^2
\end{align*}
Also from Lemma~\ref{lem (BARBAR-RL sec) : bound on the epoch length}, we know a lower bound on $N_m$. Therefore we have
\begin{align*}
    16*128^2\lambda_1\lambda_2|\calS|^2H^4 |\calA|^2\ln(10T|\Pi_{1/T}|/\delta_{overall})/(\epsilon_m)^2 
    \leq  \frac{16\lambda_1\lambda_2}{\ln(10T|\Pi_{1/T}|/\delta_{overall})} (\frac{C_m^p}{\Gamma_m-1})^2
\end{align*}
Rearranging the above inequality we get 
\begin{align*}
    \Gamma_m - 1 \leq C_m^p \epsilon_{m}/(128 H^2|\calS||\calA|\ln(10T|\Pi_{1/T}|/\delta_{overall})
\end{align*}
\end{proof}

\subsection{Proof for main theorem}

\begin{proof}
    Assume $\calE_{overall}$, $\calE_{est}$ and $\calE_{unfinished}$ occur. Now we decompose the regret into
    \begin{align*}
        \Reg
        &= \sum_{m=1}^M \sum_{\pi \in \Pi} \sum_{k=1}^{\Gamma_m} \sum_{t \in E_m^k}(\mathring{V} - V_*^\pi)\one\{ \pi_t = \pi \}
            + T(V^*-\mathring{V})\\
        &\leq \sum_{m=1}^M \sum_{j \in S_m} \sum_{k=1}^{\Gamma_m} \mathring{\Delta}_j^m \tilde{n}_j^{m,k}  + \order(H)\\
        &\leq \underbrace{\frac{3}{2} \sum_{m=1}^M \sum_{j \in S_m}  \mathring{\Delta}_j^m n_j^{m,\Gamma_m}}_\textsc{non-repeat term} 
            + \underbrace{\frac{3}{2} \sum_{m=1}^M \sum_{k=1}^{\Gamma_m -1}\sum_{j \in S_m} \mathring{\Delta}_j^m n_j^{m,k} }_\textsc{repeat term} 
            + \order(H)
    \end{align*}
    where the last inequality comes from event $\calE_{overall}$.
    For convenience, denote $R_j^{m,k} = \mathring{\Delta}_j^m n_j^{m,k}$, $\beta = 512\sqrt{\lambda_1\lambda_2\ln(10T|\Pi_{1/T}|/\delta_{overall})}|\calS||\calA|H^2$ and we know by definition that $\epsilon_j \leq \beta \sqrt{1/n_j^m}$.
    \\\\
    \textbf{We first give upper bounds on term $R_j^{m,k}$ for any fixed $m,k$}. Notice that when the algorithm goes to epoch $m$, it suggests that all the sub-algorithms ran before $m$ are completed. Therefore, we will use lemmas stated in Section~\ref{sec (BARBAR_RL sec): completed subalgo} for the following proof.
    % \textbf{We first deal with the \textsc{non-repeat term} by using the similar idea in the proof of Theorem 1 \cite{gupta2019better}}. Notice that these terms are only related to the complete sub-algorithms, so we will use lemmas stated in Section~\ref{sec (BARBAR_RL sec): completed subalgo} for the following proof.
    \paragraph{Case 1:} $\rho_{m-1} < \mathring{\Delta}_j^m/64$. In this case, if $\mathring{\Delta}_j^m/2 \geq 2^{-(m-1)}$, given $\calE_{est}$ , we can use Corollary~\ref{cor: gap estimation 2} to get
    \begin{align*}
        \epsilon_j 
        \geq \frac{1}{4}\mathring{\Delta}_j^m - 3\rho_{m-1} - \frac{3}{8}2^{-(m-1)}
        \geq \left(\frac{1}{4} - \frac{3}{64} - \frac{3}{16}\right)\mathring{\Delta}_j^m
        = \frac{\mathring{\Delta}_j^m}{64}
    \end{align*}
    If $\mathring{\Delta}_j^m/2 < 2^{-(m-1)}$, then $\epsilon_j \geq \frac{\mathring{\Delta}_j^m}{64}$ trivially holds.
    
    In turn, we have $n_j^m \leq \beta/\epsilon_j^2 $ according to the definition of $n_j^m$, from which follows 
    \begin{align*}
        R_i^{m,k} \leq 64\beta \sqrt{n_j^m}
    \end{align*}
    This can be also be written as 
    \begin{align*}
        R_i^{m,k} 
        \leq \mathring{\Delta}_j^m \beta/ \epsilon_j^2 
        \leq 64^2 \mathring{\Delta}_j^m \beta/ (\mathring{\Delta}_j^m)^2
        = 64^2 \beta/ \mathring{\Delta}_j^m
        \leq 64^2 \beta/ \min_{\pi \in \Pi_{1/T}} \mathring{\Delta}_\pi
    \end{align*} 
    \paragraph{Case 2: } $\rho_{m-1} \geq \mathring{\Delta}_j^m/64$. We again use the upper bound of $n_j^m \leq \beta^2/\epsilon_{m}^2$
    \begin{align*}
        R_i^{m,k} \leq 96 \beta^2 \rho_{m-1}/\epsilon_m^2 = 96 \beta^2 \rho_{m-1} 2^{2m}
    \end{align*}
    By combining these two cases, we have
    \begin{align*}
        R_j^{m,k} 
        \leq 64\beta \min\left\{\sqrt{n_j^m}, \frac{64}{\min_{\pi \in \Pi_{1/T}}\mathring{\Delta}_\pi}\right\} + 96 \beta^2 \rho_{m-1}/\epsilon_m^2
    \end{align*}
    \textbf{Secondly, we deal with the \textsc{non-repeat term}}. By summing $R_j^{m,k}$ over all policy sets for $k = \Gamma_m$, we get
    \begin{align*}
         &\sum_{m=1}^M \sum_{j \in S_m} \mathring{\Delta}_j^m n_j^{m,\Gamma_m}\\
         &\leq 64\beta \sum_{m=1}^M \min\left\{\sqrt{\log T N_m},\frac{64\log T}{\min_{\pi \in \Pi_{1/T}}\mathring{\Delta}_\pi} \right\} + 96 \beta^2 (\log T) \sum_{m=1}^M \rho_{m-1}2^{2m} \\
         &\leq 64\beta (\log T) \min\left\{\sqrt{T},\frac{64\log T}{\min_{\pi \in \Pi_{1/T}}\mathring{\Delta}_\pi} \right\}  + 96 \beta^2 (\log T) \sum_{m=1}^M \rho_{m-1}2^{2m}\\
        %  &\leq \tilde{\order}\left( |\calS|^{2}|\calA|^{3/2}H^{3/2} \ln(1/\delta_{overall}) \sqrt{T} \right)
        %     + \tilde{\order}\left( |\calS|^4|\calA|^3 H^3 \ln(1/\delta_{overall})^2 (C^p+C^r)\right) \\
        % &\leq \tilde{\order}\left( |\calS|^{2}|\calA|^{3/2}H\min\{ H^{1/2}, |\calS|^{1/2}|\calA|^{1/2} \} \ln(1/\delta_{overall}) \sqrt{T} \right) \\
        %     & \quad + \tilde{\order}\left( |\calS|^{4}|\calA|^{3}H^3\min\{ H, |\calS||\calA| \} \ln(1/\delta_{overall})^2 (C^p+C^r)\right) 
        &\leq \tilde{\order}\left( |\calS|^{2}|\calA|^{3/2}H^2\min\{ H^{1/2}, |\calS|^{1/2}|\calA|^{1/2} \} \ln(1/\delta_{overall}) \min\left\{\sqrt{T},\frac{1}{\min_{\pi \in \Pi_{1/T}}\mathring{\Delta}_\pi} \right\} \right) \\
            & \quad + \tilde{\order}\left(|\calS||\calA| \ln(1/\delta_{overall}) (HC^p+C^r) \right)\\
        &= \tilde{\order}\left( |\calS|^{2}|\calA|^{3/2}H^2\min\{ H^{1/2}, |\calS|^{1/2}|\calA|^{1/2} \} \ln(1/\delta_{overall}) \min\left\{\sqrt{T},\frac{1}{\min_{\pi \in \Pi}\Delta_\pi} \right\} \right) \\
            & \quad + \tilde{\order}\left(|\calS||\calA| \ln(1/\delta_{overall}) (HC^p+C^r) \right)
    \end{align*}
    The last equation comes from the fact that $\Pi_{1/T}$ is $1/T$-net of policy and $\sqrt{T} > \frac{1}{\min_{\pi \in \Pi_{1/T}}\mathring{\Delta}_\pi}$ when $\min_{\pi \in \Pi_{1/T}}\mathring{\Delta}_\pi < o(\sqrt{1/T})$. 

    Here the result of $\sum_{m=1}^M \rho_{m-1}2^{2m}$ comes from the following,
    \begin{align*}
        \sum_{m=1}^M \beta^2 \rho_{m-1}/\epsilon_m^2
        &= \sum_{m=1}^M \beta^2 \sum_{s=1}^{m-1}4^m \frac{8\lambda_1\lambda_2(HC_s^p+C_s^r)}{16^{m-1-s}N_s} \\
        &= 8\lambda_1\lambda_2\beta^2\sum_{s=1}^M (HC_s^p+C_s^r)  \sum_{m=s}^{M}4^m \frac{1}{16^{m-1-s}N_s} \\
        &\leq 8\lambda_1\lambda_2\beta^2\sum_{s=1}^M (HC_s^p+C_s^r)  \sum_{m=s}^{M}4^m \frac{4^{-s}}{16^{m-1-s}\beta^2}\\
        & = 32 \lambda_1\lambda_2 \sum_{s=1}^M (HC_s^p+C_s^r)  \sum_{m=s}^{M} \frac{4^{m-1-s}}{16^{m-1-s}}\\
        & = \tilde{\order}\left(|\calS||\calA|\ln(1/\delta_{overall}) (HC^p+C^r) \right)
    \end{align*}
where the first equality use changing order of summation techniques and the second inequality comes from the lower bound of $N_s$ in Lemma~\ref{lem (BARBAR-RL sec) : bound on the epoch length}.

    \textbf{Thirdly, we consider the \textsc{repeat term}.} From the previous analysis, we have
    \begin{align*}
        \sum_{m=1}^M \sum_{k=1}^{\Gamma_m -1} \sum_{j \in S_m} \mathring{\Delta}_j^m n_j^{m,k}
        \leq 64\beta\sum_{m=1}^M \sum_{k=1}^{\Gamma_m -1} \sqrt{ (\log T) N_m} + \sum_{m=1}^M (\Gamma_{m'} -1) 96 \beta^2 (\log T)  \rho_{m-1}2^{2m}
    \end{align*}
    \\
    First, given $\calE_{unfinished}$, we can bound the first term by 
    \begin{align*}
        64\beta \sum_{m=1}^M \sum_{k=1}^{\Gamma_m -1} \sqrt{\log T} C_{m,k}^p \frac{16\sqrt{\lambda_1\lambda_2}}{\sqrt{\ln(10T|\Pi_{1/T}|/\delta_{overall})}}
        \leq \tilde{\order}\left( H^2|\calS|^{2}|\calA|^2\ln(1/\delta_{overall}) C^p\right)
    \end{align*}
    Then, by Lemma~\ref{lem (BARBAR_RL sec): upper bound for repeated epoch numbers}, we can bound the first term by bounding the $\Gamma_m-1$ as below
    \begin{align*}
        &\beta^2 (\log T) \sum_{m=1}^M (\Gamma_m -1) \rho_{m-1}2^{2m} \\
        & \leq \beta^2 (\log T) \sum_{m=1}^M \frac{C_{m}^p}{H^2|\calS||\calA|\ln(10T|\Pi_{1/T}|/\delta_{overall})} \rho_{m-1}2^{2m}\\
        & \leq  \frac{\log T}{H^2|\calS||\calA|\ln(10T|\Pi_{1/T}|/\delta_{overall})}\left(\sum_{m'=1}^M C_m^p \right) \left( \sum_{m=1}^M \beta^2  \sum_{m' \in M} \rho_{m-1}2^{2m} \right) \\
        & \leq  \frac{C^p (\log T)^2}{H^2|\calS||\calA|\ln(10T|\Pi_{1/T}|/\delta_{overall})} \left( \beta^2 \sum_{m=1}^M \rho_{m-1}2^{2m} \right) \\
        % &\leq \tilde{\order}\left(|\calS|^2|\calA|^2\ln(1/\delta_{overall}) C^p(C^p+C^r)\right)
        &\leq \tilde{\order}\left(\frac{1}{H^2} C^p(HC^p+C^r)\right)
        % & \leq \order(C^p(C^p+C^r) (\log T)^5 \log(1/\delta)|\calS|^3)
    \end{align*}
Combing all the upper bounds, we get the final result.
\end{proof}

\subsection{Relationship between PolicyGapComlexity and the GapCompelxity in \citet{DBLP:conf/nips/SimchowitzJ19}}
\label{app: details on gap complexity}

In the main paper, we assume a single starting states. Here, in order to make a comparison, we remove this assumption and assume a starting distribution over all states. As stated in the \textbf{Related Work} section, the most common GapComplexity used in reinforcement learning is in the following form. Note that to aid the exposition, we omit other states and actions dependency below.
\begin{align*}
    & \text{gap}_h(s,a) = V_h^*(s) - Q^*_h(s,a),\\
    & \text{GapComplexity} = \frac{1}{\min_{s,a,h}  \text{gap}_h(s,a)}
\end{align*}
To get an intuition about its relation to policy gap $\Delta_\pi$, consider the optimal policy $\pi^*$ and the second optimal policy $\pi'$. If there is a tie, we just arbitrarily choose two policies with closest behavior. Define
\begin{align*}
    \mathcal{H}_{identical} = \{ h| \forall h' \in [0,h-1], \forall s\in \calS_{h'}, \pi^*(s) = \pi'(s) \}
\end{align*}
where $\calS_h = \{s \in \calS | \max_{\pi \in \Pi} \text{Prob} \left( \text{$\pi$ visits $s$ at $h$} \right) > 0\}$ and $\calS_0 = \emptyset$. So $\mathcal{H}_{identical}$ is a collection of steps, before which, the optimal policy $\pi^*$ and the second optimal policy $\pi'$ are unidentifiable. 
Note that $h=1$ is always included in $\mathcal{H}_{identical}$. Now we have
\begin{align*}
    \Delta_{\pi'}
    & = V^* - V_*^{\pi'} \\
    & = \max_{h \in \mathcal{H}_{identical}} \sum_{s \in \calS} \text{Prob} \left( \text{$\pi^*$ visits $s$ at $h$} \right) \left(V^*_h(s) - Q_{*,h}^{\pi'}(s,\pi'(s))\right) \\
    &\geq \max_{h \in \mathcal{H}_{identical}} \sum_{s \in \calS} \text{Prob} \left( \text{$\pi^*$ visits $s$ at $h$} \right) \left(V^*_h(s) - Q_{h}^*(s,\pi'(s))\right) \\
    & \geq \min_{s,a,h} \text{gap}_h(s,a) \max_{h \in \mathcal{H}_{identical}} \sum_{s \in \calS} \text{Prob} \left( \text{$\pi^*$ visits $s$ at $h$} \right) \one\{ \pi^*(s) \neq \pi'(s) \} 
\end{align*}
It is easy to see that $\max_{h \in \mathcal{H}_{identical}} \sum_{s \in \calS} \text{Prob} \left( \text{$\pi^*$ visits $s$ at $h$} \right) \one\{ \pi^*(s) \neq \pi'(s) \} $ is positive due to the definition of  $\mathcal{H}_{identical}$. 
\\\\
Recall the the PolicyGapComplexity is defined as $\frac{1}{\Delta_{\pi'}}$, so we have 
\begin{align*}
    \text{PolicyGapComplexity}
    &\leq \frac{1}{\max_{h \in \mathcal{H}_{identical}} \sum_{s \in \calS} \text{Prob} \left( \text{$\pi^*$ visits $s$ at $h$} \right) \one\{ \pi^*(s) \neq \pi'(s) \} } \frac{1}{\min_{s,a,h} \text{gap}_h(s,a)} \\
    & \leq \frac{\text{GapComplexity}}{\max_{h \in \mathcal{H}_{identical}} \sum_{s \in \calS} \text{Prob} \left( \text{$\pi^*$ visits $s$ at $h$} \right) \one\{ \pi^*(s) \neq \pi'(s) \} }
\end{align*}
Therefore, with respect to the gap term, the PolicyGapComplexity and the GapComplexity are close when $\max_{h \in \mathcal{H}_{identical}} \sum_{s \in \calS} \text{Prob} \left( \text{$\pi^*$ visits $s$ at $h$} \right) \one\{ \pi^*(s) \neq \pi'(s) \}$ is large.
\\\\
Because step $h=1$ is always included in $\mathcal{H}_{identical}$, so one nontrivial case satisfying the above condition is that the starting states are uniformly chosen from some subset of states. It is easy to see that the single starting states is also one of the special cases. Besides, there are also many other cases satisfying the above condition, for example, a MDP that starts from various states and always concentrates on some states with equal chances in later steps included in $\mathcal{H}_{identical}$.
\\\\
Finally, whether the PolicyGapComplexity-dependent bound can also get some refined dependency on $|\calS|,|\calA|,H$ like the GapComplexity-dependent bound in \citet{xu2021fine} in some special cases remains further investigation.

% One can see that 
% \begin{align*}
%     |V_*^\pi - V_*^{\pi'}|
%     \geq \sum_{s \in \calS} \text{Prob} \left( \text{$\pi$ visits visit $s$ at $h$} \right) \left(Q^*_h(s,\pi(s)) - Q^*_h(s,\pi'(s)\right) 
% \end{align*}\simon{Why here is $Q^*$ instead of $Q^\pi$ or $Q^{\pi'}$? Am I missing something? Maybe we just need to compare $\pi^*$ and the second best policy?}
% If $\pi'$ is the optimal policy $\pi^*$, then we have
% \begin{align*}
%     \frac{1}{\min_{\pi \in \Pi} \Delta_\pi}
%     \leq \frac{1}{\min_{s \in \calS}\text{Prob}\left( \text{$\pi$ visits $s$ at $h$} \right) } \frac{1}{\min_{s,a,h} \text{gap}_h(s,a)}
% \end{align*}
% Thus, in general, the PolicyGapComplexity is worse than the GapCompelxity defined in \citet{DBLP:conf/nips/SimchowitzJ19}, but they will be close if $\text{Prob}\left( \text{$\pi$ visits $s$ at $h$} \right)$ is large. Since we are considering the deterministic policies, so this suggests that for any $(s,a),s',s''$, non-zero transition probability $P(s'|s,a)$ and $P(s''|s,a)$ are close.
% \simon{The last point is bit vague to me. I think it is possible that the occupancy measure of two states are close but does not require for any $(s,a),s',s''$, non-zero transition probability $P(s'|s,a)$ and $P(s''|s,a)$ are close.}
\section{Meta-algorithm and Results for cheated Adversary}
\label{sec: app-mainAlgo2}

\begin{algorithm}[H] 
\caption{ \textsc{Brute-force-Policy-Elimination-RL} }
\begin{algorithmic}[1] 
\STATE \textbf{Input:} time horizon $T$, confidence $\delta_{overall}$
\STATE Construct a $1/T$-net for non-stationary policies, denoted as $\Pi_{1/T}$.
\STATE Initialize $S_1 = 0, \Pi^1 = \Pi$. And for $j \in \log T$, initialize $\epsilon_j = 2^{-j}. \epsilon_{sim}^j = \epsilon_j/128 $
\STATE Set $\lambda_1=6 |\calS||\calA| log(H^2|\calS||\calA|/\epsilon_{sim}) $ and $\lambda_2= 12 \ln(8T/\delta_{overall})$ 
\FOR{ epoch $m = 1,2,\ldots$}
    \STATE Set $\delta^m = \delta_{overall}/(5T)$
    \STATE Set $F^m = \frac{8|\calS|^2H^4 |\calA|^2\ln(2|\Pi^m|/\delta^m)}{(\epsilon_{sim}^m)^2}$ 
    \STATE Set $N_m = 2 \lambda_1\lambda_2 F^m$ and $T_m^s = T_{m-1}^s+N_{m-1}$
    \STATE Initialize a sub-algorithm $\textsc{EstAll}^m = \text{EstAll}(\epsilon_{sim}^m,\Pi^m,\delta^m,F^m)$ 
    \FOR{ $t=T_m^s, T_m^s+1, \ldots,  T_{m}^s+N_{m}-1$} \label{line: start an epoch}
        \STATE Play the policy according the awaiting $\textsc{EstAll}^m.\textsc{continue}$. Then continue running $\textsc{EstAll}^m$ until the next \textsc{rollout} is met. (If no more \textsc{rollout} needed, then just uniformly play one ) 
    \ENDFOR    
    \IF{ $\textsc{EstAll}^m$ is unfinished}
        \STATE Set $T_m^s = T_{m}^s+N_{m}$ and repeat the whole process from line 9. ~~~~ $\triangleright$ So each repeat is a sub-epoch.
    \ELSE
        \STATE Obtain $\hat{r}_m(\pi)$ for all $\pi$.
    \ENDIF
    \STATE Update the active policy set \label{line: eliminate condition}
    % $\Pi^{m+1} \leftarrow \{\pi | \max_{\pi' \in \Pi^m}\hat{r}_m(\pi') - \hat{r}_m(\pi) \leq \lambda_1\lambda_2H^{3/2}\sqrt{|\calS||\calA|\ln(10T|\Pi_{1/T}|/\delta_{overall})}\frac{8\sqrt{HT}}{N_m} + \frac{1}{8}\epsilon_m\}$ 
    \begin{align*}
        \Pi^{m+1} \leftarrow \{\pi | \max_{\pi' \in \Pi^m}\hat{r}_m(\pi') - \hat{r}_m(\pi) \leq 8\lambda_1\lambda_2H^2\sqrt{|\calS||\calA|\ln(10T|\Pi_{1/T}|/\delta_{overall})T}/N_m + \frac{1}{8}\epsilon_m\}
    \end{align*}
    
\ENDFOR
\end{algorithmic}
\end{algorithm}

% In previous \textsc{BARBAR-RL}, we use the randomness of algorithm to ensure that the corruption effect on each policy estimation is proportional to its estimated. But in this " cheated" adversary setting, the randomness of policy no longer works. So here we use a brute-force policy elimination, which is based on the traditional policy elimination but enlarge the confidence range by $\tilde{\order}(\sqrt{HT})$ when doing the elimination. as ( See Line~\ref{line: eliminate condition}). Therefore, the best policy will never be eliminated as long as $C^p+C^r \leq \tilde{\order}(\sqrt{HT})$. But such brute method will lead to a regret in term s of $(C^r)^2$ instead of $C^r$.
% \\\\
% Another problem is again, to get a uniform estimation on each policy in reinforcement learning. Here we use the same solution as in \textsc{BARBAR-RL}, which is, simply restart the sub-algorithm when it is unfinished.

\begin{theorem}
\label{them  (PolicyElim ssc) ：regret for PolicyElim}
The regret is upper bounded by
\begin{align*}
    % \tilde{\order}\left( |\calS|^2|\calA|^{3/2}H^{3/2}\sqrt{HT} + |\calS|^2|\calA|^{2}(C^p)^2 +   \frac{(C^r)^2}{|\calS|^4|\calA|^{3}H^{3}}\right)
    \Reg &\leq \tilde{\order}\left(|\calS|^2|\calA|^{3/2}H^2 \min\{\sqrt{H},\sqrt{|\calS||\calA|}\} \ln(1/\delta_{overall})\sqrt{T}\right) \\
        &\quad + \tilde{\order}\left( \frac{( C^r)^2}{H^3|\calS||\calA|}
        +H|\calS||\calA|(C^p)^2\right) 
\end{align*}
\end{theorem}

\paragraph{Remark} In Section 2.2 in \cite{bogunovic2020stochastic}, they proved that in order to get $\tilde{\order}(\sqrt{HT})$, the corruption terms can go as low as $\tilde{\Omega}(\frac{C^2}{\log C})$ for the linear bandits. Therefore, we conjecture that $\tilde{\order}((C^r+C^p)^2)$ term is also unavoidable in our setting.

\subsection{Regret Analysis for Theorem~\ref{them  (PolicyElim ssc) ：regret for PolicyElim}}

For convenience, we rearrange this upper bound a little bit. So now our target is to show the follows. 
\begin{align*}
    % \Reg \leq \tilde{\order}\left(|\calS|^2|\calA|^{3/2}H^{3/2}\sqrt{HT} + |\calS|^2|\calA|^{2}(C^p)^2 +  \left( \frac{(C^p+C^r)^2}{|\calS|^4|\calA|^{3}H^{3}}\right)\right)
    \Reg &\leq \tilde{\order}\left(|\calS|^2|\calA|^{3/2}H^2 \min\{\sqrt{H},\sqrt{|\calS||\calA|}\} \ln(1/\delta_{overall})\sqrt{T}\right) \\
        &\quad + \tilde{\order}\left( \frac{(HC^p + C^r)^2}{H^3|\calS||\calA|\ln(|\Pi_{1/T}|)}
        +H\frac{\ln(1/\delta_{overall})}{\ln(|\Pi_{1/T}|/\delta_{overall})}|\calS||\calA|(C^p)^2
    \right) 
\end{align*}
We only need to consider the case that $C^r + HC^p \leq H^{2}\sqrt{|\calS||\calA|\ln(|\Pi_{1/T}|)T}$, otherwise we will get a trivial linear regret.  
\\\\
It easy to see that the following events sill holds with at least $1 - \delta_{overall}$ probability,
\begin{align*}
    & \calE_{overall}: = \left\{\forall m, \forall k \in [\Gamma_m]: \tilde{n}^{m,k} \in [\frac{1}{2}n^m, \frac{3}{2} n^m ] \right\} \\
    &\calE_{est}: = \left\{\forall m,\pi \in \Pi^{m}: |\hat{r}^m(\pi) - V_*^\pi| \leq 2\lambda_1\lambda_2\frac{2(HC_{m,\Gamma_m}^p+C_{m,\Gamma_m}^r)}{N_m} + \frac{1}{16}\epsilon_m \right\}\\
    &\calE_{unfinished}: = \left\{\forall m, \forall k \in [\Gamma_m]: C_{m,k}^p 
        \geq \frac{1}{4}\sqrt{\frac{\ln(10T|\Pi|/\delta_{overall})}{\lambda_1\lambda_2}N_m} \right\} \text{ and } \calE_{overall} 
\end{align*}
Notice here we will permanently eliminate a policy instead of maintaining different subset of policies, therefore, in $\calE_{est}$, all the active policies have same levels of estimation. Next we show that given the above events, we will never eliminate the best policy from the active policy set $\Pi^{m+1}$.

Again we use the following notations $\mathring{\pi} = \argmax_{\pi \in \Pi_{1/T}} V_*^\pi$, $\mathring{V} = V_*^{\mathring{\pi}}$ and $\mathring{\Delta}_{\pi} = \mathring{V} - V_*^\pi$.

\begin{lemma}
\label{lem (PolicyElim sec) : maintain best policy}
    For any epoch $m$, we always have $\mathring{\pi} \in \Pi^{m}$.
\end{lemma}
\begin{proof}
    Given $\calE_{est}$, let $\hat{\pi}_m = \argmax_{\pi' \in \Pi^m}\hat{r}_m(\pi') $, we know that 
    \begin{align*}
        \hat{r}_m(\hat{\pi}_m) - \hat{r}_m(\mathring{\pi})
        &\leq V_*^{\hat{\pi}_m} - \mathring{V} + 4\lambda_1\lambda_2\frac{2(HC_{m,\Gamma_m}^p+C_{m,\Gamma_m}^r)}{N_m} + \frac{1}{8}\epsilon_m\\
        &\leq 4\lambda_1\lambda_2\frac{2(HC_{m,\Gamma_m}^p+C_{m,\Gamma_m}^r)}{N_m} + \frac{1}{8}\epsilon_m\\
        &\leq 8\lambda_1\lambda_2H^{2}\sqrt{|\calS||\calA|\ln(|\Pi_{1/T}|)T}/N_m + \frac{1}{8}\epsilon_m
*    \end{align*}
    where the last inequality comes from the assumption that $C^r + HC^p \leq H^{2}\sqrt{|\calS||\calA|\ln(|\Pi_{1/T}|)T}$. Now by the elimination condition in Line~18
    % \ref{line: eliminate condition}
    , we can get our target result.
\end{proof}
Then we can upper bounded $\max_{\pi \in \Pi^m}  \Delta_\pi$ as follows
\begin{lemma}
\label{lem (PolicyElim sec) : epoch regret}
    For any active policy set $\Pi^m$, we have
    \begin{align*}
        \max_{\pi \in \Pi^m}  \Delta_\pi 
        \leq \tilde{\order}\left(|\calS|^2|\calA|^{3/2}H^{3/2}(\frac{1}{\sqrt{N_{m}}} + \frac{\sqrt{HT}}{N_{m}}) \right)
    \end{align*}
\end{lemma}
\begin{proof}
    Let $\pi' = \argmax_{\pi \in \Pi^{m+1}}  \Delta_\pi$
    \begin{align*}
    \mathring{\Delta}_{\pi'}
    & \leq \mathring{V} - V_*^{\pi'}\\
    &\leq \hat{r}_m(\mathring{\pi}) - \hat{r}_m(\pi')+ 4\lambda_1\lambda_2\frac{2(HC_{m,\Gamma_m}^p+C_{m,\Gamma_m}^r)}{N_m} + \frac{1}{8}\epsilon_m\\
    & \leq 8\lambda_1\lambda_2H^{2}\sqrt{|\calS||\calA|\ln(|\Pi_{1/T}|)T}/N_m + \frac{1}{4}\epsilon_{m+1}\\
    & = \tilde{\order}\left(|\calS||\calA|\ln(1/\delta_{overall})H^{2}\sqrt{|\calS||\calA|\ln(|\Pi_{1/T}|)}\frac{\sqrt{T}}{N_{m+1}} + |\calS|^{3/2}|\calA|^{3/2}H^2\sqrt{\frac{\ln(1/\delta_{overall})\ln(10T|\Pi_{1/T}|/\delta_{overall})}{N_{m+1}}} \right)\\
    &\leq \tilde{\order}\left( |\calS|^{3/2}|\calA|^{3/2}H^2\ln(1/\delta_{overall})\sqrt{\ln(|\Pi_{1/T}|)} (\sqrt{T}+\sqrt{\frac{1}{N_{m+1}}})\right)\\
    &\leq \tilde{\order}\left( |\calS|^2|\calA|^{3/2}H^2 \min\{\sqrt{H},\sqrt{|\calS||\calA|}\} \ln(1/\delta_{overall})\left(\sqrt{T}+\sqrt{\frac{1}{N_{m+1}}}\right) \right)
\end{align*}
Here the second inequality comes from Lemma~\ref{lem (PolicyElim sec) : maintain best policy}. The third inequality comes from the elimination condition in Line 18
% Line~\ref{line: eliminate condition} 
and the assumption that the assumption that $C^r + HC^p \leq H^{2}\sqrt{|\calS||\calA|\ln(|\Pi_{1/T}|)T}$. Replace the value of $\epsilon_m$ in the term of $N_m$ we get the target result.
\end{proof}
Now given $\calE_{overall}$, we again have regret that 
\begin{align*}
    \Reg
    \leq \underbrace{\frac{3}{2} \sum_{m=1}^M (\max_{\pi \in \Pi^m}  \Delta_\pi) N_m}_\textsc{non-repeat term} 
        + \underbrace{\sum_{m=1}^M \sum_{k=1}^{\Gamma_m -1} N_m }_\textsc{repeat term} 
\end{align*}
First, we deal with the \textsc{non-repeat term}. By applying Lemma~\ref{lem (PolicyElim sec) : epoch regret}, we have
\begin{align*}
    \sum_{m=1}^M (\max_{\pi \in \Pi^m}  \Delta_\pi) N
    &\leq \sum_{m=1}^M\tilde{\order}\left(|\calS|^2|\calA|^{3/2}H^2 \min\{\sqrt{H},\sqrt{|\calS||\calA|}\} \ln(1/\delta_{overall})\left(\sqrt{T}+\sqrt{\frac{1}{N_{m+1}}}\right)\right)\\
    &\leq \tilde{\order}\left(|\calS|^2|\calA|^{3/2}H^2 \min\{\sqrt{H},\sqrt{|\calS||\calA|}\} \ln(1/\delta_{overall})\sqrt{T}\right)
\end{align*}
Next, we deal with the \textsc{repeat term}. By $\calE_{unfinished}$, we have
\begin{align*}
    \sum_{m=1}^M (\max_{\pi \in \Pi^m}  \Delta_\pi) N_m
    \leq H \sum_{m=1}^M \sum_{k=1}^{\Gamma_m -1} N^m 
    &\leq H|\calA||\calS| \frac{\ln(1/\delta_{overall})}{\ln(10T|\Pi_{1/T}|/\delta_{overall})}\sum_{m=1}^M \sum_{k=1}^{\Gamma_m -1} (C_{m,k}^p)^2 \\
    &\leq H|\calA||\calS| (C^p)^2 
\end{align*}
\section{Analysis for \textbf{EstAll} Sub-algorithm} 
\label{sec: app-EstAll}

\subsection{Preliminaries}
We define the set of episodes that the learner interacts with environment as $\calI_{est}$ and the total corruption included these episodes as $C_{est}^{r(p)} = \sum_{t \in \calI_{est}} c_t^{r(p)}$.

\subsection{Key results}

\begin{theorem}[Sample complexity restated here]
% \label{them (EstAll sec): maximun interaction number of EstAll}
     Suppose $F \geq \frac{8|\calS|^2H^4 |\calA|^2\ln(2|\Pi|/\delta_{est})}{\epsilon_{est}^2}$ and $\tau \geq 6$. Under the corruption assumption $C_{est}^p \leq \frac{\epsilon_{est}F}{2|\calS||\calA|H^2}$, with probability at least $1- \delta_{est}$, the algorithm interacts with environment at most
     \begin{align*}
         |\calS||\calA|F\tau log(H^2|\calS||\calA|/\epsilon_{est})
     \end{align*}
     times.
     Note, if the algorithm interacts with environment more than the above number of times, then with probability at least $1- \delta_{est}$, $C_{est}^p > \frac{\epsilon_{est}F}{2|\calS||\calA|H^2}$
\end{theorem}
\begin{proof}
By Lemma~\ref{lem (EstAll sec): condition when add new pi}, we know that with probability at least $1 - \delta_{est}$, for any fixed state-action pair $(s,a)$,
% Line~\ref{line (EstAll sec): add pi} 
Line 7 in Algorithm~\ref{algo: EstAll} will fail at most $\log_2(H^2|\calS||\calA|/\epsilon_{est})$ times by doubling from $\frac{\epsilon_{est}}{H|\calS||\calA|}$ to $H$. So the maximum number of policies that will be added into policy set $\Pi_\calD$ is at most $\log_2(H^2|\calS||\calA|/\epsilon_{est})|\calS||\calA|$ . Now because for each policy added into $\Pi_\calD$, we will greedily sample $F\tau$ times according to Algorithm~\ref{algo: ROLLOUT}, so the total interaction time is at most $\log_2(H^2|\calS||\calA|/\epsilon_{est})|\calS||\calA|F\tau$ times.
\end{proof}

\begin{theorem}[Estimation correctness restated here]
% \label{them (EstAll sec): value estimation (key) }
     Suppose $F \geq \frac{8|\calS|^2H^4 |\calA|^2\ln(2|\Pi|/\delta_{est})}{\epsilon_{est}^2}$ and $\tau \geq 6$. Then for all $\pi \in \Pi$, with probability at least $1 - \delta_{est}$,
     \begin{align*}
          \big| \hat{r}(\pi) - V^{\pi}(s_1) \big| \leq (1+\tau)\epsilon_{est} + (HC_{est}^p+C_{est}^r)/F
     \end{align*}
\end{theorem}
\begin{proof}
By definition, $\hat{r}(\pi) = \frac{1}{F}\sum_{i=1}^F r(z_i^\pi)$ and  $\{r(z_i^\pi) \}_{i=1}^F$ is a sequence of independent random variables. We denote its expected value $\E[r(z_i^\pi)]$ as $\{V^\pi_i \}_{i=1}^F$. Here $V_i$ is not a real existing value function but an ``average value function'' whose rewards and transition functions are the average of rewards and transition functions generated by the MDPs under different times (so some are corrupted). Now we can use Hoeffding's inequality to bound $\big| \hat{r}(\pi) - \frac{1}{F}\sum_{i=1}^F V^\pi_i \big|$.
\\\\
For those $\pi \in \Pi_\calD$,
\begin{align*}
    \prob\left[ \big| \hat{r}(\pi) - \frac{1}{F}\sum_{i=1}^F V^\pi_i \big| \leq \epsilon_{est}\right]
    \geq 1 - 2\exp(-2 F\epsilon_{est}^2/H^2) 
    \geq 1 - \delta_{est}/2|\Pi| 
\end{align*}
For those $\pi \notin \Pi_\calD$, if none of then are failed, we again have 
\begin{align*}
    \prob\left[ \big| \hat{r}(\pi) - \frac{1}{F}\sum_{i=1}^F V^\pi_i \big| \leq \epsilon_{est}\right]
    \geq 1 - \delta_{est}/2|\Pi|
\end{align*}
Then because at each $(s,a)$, the policy \textit{fails} at most $\epsilon_{est} \tau F/H|\calS||\calA|$, there will be at most $\tau \epsilon_{est} F/H$ trajectories with \textit{Fails}. Each failed trajectory will cause at most $H$ rewards, therefore,
\begin{align*}
    \prob\left[ \big| \hat{r}(\pi) - \frac{1}{F}\sum_{i=1}^F V^\pi_i \big| \leq (1+\tau)\epsilon_{est}\right]
    \geq 1 - \delta_{est}/2|\Pi|
\end{align*}
Now we can decompose our target result into,
\begin{align*}
    \big| \hat{r}(\pi) - V^{\pi} \big|
    \leq  \big| \hat{r}(\pi) -\frac{1}{F}\sum_{i=1}^F V^\pi_i \big| + \big|\frac{1}{F}\sum_{i=1}^F V^\pi_i - V^\pi \big|
\end{align*}
The first term can be upper bounded by the previous results. The second term can be upper bounded by lemma~\ref{lem (EstAll sec): transition corruption effect value function}. 
\\
Finally, by taking a union bound over all policies in $\Pi$, we get our target result.
\end{proof}

\subsection{Detailed Analysis}
\subsection{Notations}
For convenience, we write $F$ instead of $F_{est}$ in this section.

\subsubsection{Main Lemmas}
\paragraph{Claim 1}  For any fixed $\pi$, each of the trajectories in $\{z_i^\pi\}_{i \in [F]}$ is independent to each other due to the property of MDP.

\begin{definition}
Define $f^\pi(s,a)$ as the random variable which is the total number of times a trajectory induced by $\pi$ visits $(s,a)$ with respect to the underlying MDP $\calM$ and then define its expectation as 
\begin{align*}
    \E[f^{\pi}(s,a)] = \mu^\pi(s,a)
\end{align*}
For any policy set $\Pi$, we define the following $\mu^{\Pi}_{\max}$ 
\begin{align*}
    \mu_{\max}^{\Pi}(s,a) = \max_{\pi \in \Pi} \mu^\pi(s,a).
\end{align*}
This can be leveraged to compute a lower bound on the expected number of times of visiting $(s,a)$ after rolling out each $\pi$ in $\Pi$ once.
\end{definition}

\begin{lemma}
\label{lem (EstAll sec): doubling expected visiting time}
    Under the assumption of $C_{est}^p \leq \frac{\epsilon_{est}F}{2|\calS||\calA|H^2}$ . For any fixed policy $\pi$, let $\Pi_\calD$ be an exploration set of policies before simulating $\pi$. Then when $ \mu^\pi(s,a) \in\left[\frac{\epsilon_{est}}{|\calS||\calA|H},2\mu_{\max}^{\Pi_\calD}(s,a) \right]$, $\mu_{\max}^{\Pi_\calD}(s,a) \geq \frac{\epsilon_{est}}{|\calS||\calA|H}$,$F \geq \frac{8|\calS|^2H^4 |\calA|^2\ln(2|\Pi|/\delta_{est})}{\epsilon_{est}^2}$ and $\tau \geq 6$, we have with probability at least $ 1- \frac{\delta_{est}}{|\Pi|}$
    \begin{align*}
        \sum_{i=1}^F \underbrace{|\{(s,a)\text{ or } Fail(s,a,i) \text{ included in } z_i^\pi\}| }_\text{\text{total number of times $z_i^\pi$ visited $(s,a)$}}
        < |\calD_{s,a}| + \frac{\tau\epsilon_{est}}{|\calS||\calA|H}F 
    \end{align*}
\end{lemma}
\begin{proof}
    First, we are going to get the high probability lower bound on $|\calD_{s,a}|$. Denote $\sum_{h=1}^H  \one \{\pi'' \text{ visits }(s,a) \text{ at layer h during the rollout } j \}$ as $X_j$, where $\pi'' = \argmax_{\pi \in \Pi^\calD} \mu^\pi(s,a)$. We have 
    \begin{align*}
        |\calD_{s,a}|
        = \sum_{j=1}^{F\tau} \sum_{\pi' \in \Pi_\calD} \sum_{h=1}^H 
            \one \{\pi' \text{ visit }(s,a) \text{ at layer h during the rollout } j \}
        \geq \sum_{j=1}^{F\tau} X_j.
    \end{align*}
    Note that $\{ X_{j} \}$ is a sequence of independent random variable with each $X_{j} \in [0,H]$. We denote $\E[X_{j}]$ as $\mu_{j, rollout}^{\pi''}(s,a)$. From the corruption assumption $C_{est}^p \leq \frac{\epsilon_{est}F}{2|\calS||\calA|H^2}$ and by corollary~\ref{lem (EstAll sec): transition corruption effect visiting prob}, we have
    \begin{align}
    \label{eq: 1}
        |\frac{1}{F\tau}\sum_{j=1}^{F\tau} \mu_{j, rollout}^{\pi''}(s,a) - \mu_{\max}^{\Pi_\calD}(s,a)| 
        \leq \frac{HC_{est}^p}{F\tau} \leq \frac{\epsilon_{est}}{2|\calS||\calA|H}
    \end{align}
    which, combined with $\mu_{\max}^{\Pi_\calD}(s,a) \geq \frac{\epsilon_{est}}{|\calS||\calA|H}$, also leads to 
    \begin{align*}
        \frac{1}{F\tau}\sum_{j=1}^{F\tau} \mu_{j, rollout}^{\pi''}(s,a) \geq \frac{\epsilon_{est}}{2|\calS||\calA|H}
    \end{align*}
    Then by using the Hoeffding's inequality, we get 
    \begin{align*}
        \prob\left[\sum_{j}^{F\tau} X_j \leq \frac{1}{2} \sum_{j=1}^{F\tau} \mu_{j, rollout}^{\pi''}(s,a) \right]
        \leq \exp\left(- \frac{2 F^2\tau^2}{F\tau H^2}(\frac{\epsilon_{est}}{4|\calS||\calA|H})^2 \right) 
        \leq \frac{\delta_{est}}{2|\Pi|}
    \end{align*}
    Therefore, we get that with probability at least $1 - \frac{\delta_{est}}{2|\Pi|}$, $|D_{s,a}| > \frac{1}{2} \sum_{j=1}^{F\tau} \mu_j^\pi(s,a))$
    \\\\
    Second, we are going to get the high probability upper bound on $\sum_{i=1}^F |\{(s,a)\text{ or } Fail(s,a,i) \text{ included in } z_i^\pi\}|$.
    Denote $|\{(s,a)\text{ or } Fail(s,a,i) \text{ included in } z_i^\pi\}|$ as $Y_i \in [0,H]$ and its expectation $\E[Y_i] = \mu_{i, sim}^\pi(s,a)$. By Claim 1, we know that each trajectory in $\{z_i^\pi\}_{i \in [F]}$ is independent to each other. Again from the corruption assumption $C_{est}^p \leq \frac{\epsilon_{est}F}{2|\calS||\calA|H^2}$ and by corollary~\ref{lem (EstAll sec): transition corruption effect visiting prob}, we have
    \begin{align}
    \label{eq: 2}
        |\frac{1}{F}\sum_{i=1}^{F} \mu_{i, sim}^\pi(s,a) - \mu^\pi(s,a)| 
        \leq \frac{HC_{est}^p}{F} \leq \frac{\epsilon_{est}}{2|\calS||\calA|H}
    \end{align} 
    which, combined with $\mu^\pi(s,a) \geq \frac{\epsilon_{est}}{|\calS||\calA|H}$, also leads to 
    \begin{align*}
        \frac{1}{F}\sum_{j=1}^{F} \mu_{i, sim}^\pi(s,a) \geq \frac{\epsilon_{est}}{2|\calS||\calA|H}
    \end{align*}
    So by using the hoeffding inequality again, we get that with probability at least $1 - \frac{\delta_{est}}{2|\Pi|}$, 
    \begin{align*}
        \sum_{i=1}^F |\{(s,a)\text{ or } Fail(s,a,i) \text{ included in } z_i^\pi\}|
        < \frac{3}{2} \sum_{i=1}^F \mu_{i, sim}^\pi(s,a)
    \end{align*}
    \\
    Finally, combine the high probability upper bound and lower bound, we have that with probability at least $1 - \frac{\delta_{est}}{|\Pi|}$
    \begin{align*}
         & \sum_{i=1}^F |\{(s,a)\text{ or } Fail(s,a,i) \text{ included in } z_i^\pi\}| - |\calD_{s,a}| \\
         & < \frac{3}{2} \sum_{i=1}^F \mu_{i, sim}^\pi(s,a) - \frac{1}{2} \sum_{j=1}^{F\tau} \mu_{j,rollout}^{\pi''}(s,a)\\
         & \leq \frac{3}{2} F\mu^\pi(s,a) - \frac{1}{2} F\tau \mu_{max}^{\Pi_\calD}(s,a) + \frac{\epsilon_{est}}{2|\calS||\calA|H}(F + F\tau) \\
         & \leq \frac{\epsilon_{est}}{2|\calS||\calA|H}(\frac{3}{2}F + \frac{1}{2}F\tau) 
         < \frac{\epsilon_{est}}{|\calS||\calA|H}F\tau
    \end{align*}
    where the second inequality comes from eq.~\ref{eq: 1},~\ref{eq: 2} and the last inequality comes form the the assumption $\mu^\pi(s,a) < 2\mu_{\max}^{\Pi_\calD}(s,a)$, $\tau \geq 6$.
\end{proof}

\begin{lemma}
\label{lem (EstAll sec): raise the expected visited number to delta_sim}
    Under the assumption of $C_{est}^p \leq \frac{\epsilon_{est}F}{2|\calS||\calA|H^2}$ . For any fixed policy $\pi$, let $\Pi_\calD$ be an exploration set of policies before simulating $\pi$. Then when $\mu^\pi(s,a) <\frac{\epsilon_{est}}{|\calS||\calA|H}$,$F \geq \frac{8|\calS|^2H^4 |\calA|^2\ln(2|\Pi|/\delta_{est})}{\epsilon_{est}^2}$ and $\tau \geq 6$, we have with probability at least $ 1- \frac{\delta_{est}}{|\Pi|}$
    \begin{align*}
        \sum_{i=1}^F |\{(s,a)\text{ or } Fail(s,a,i) \text{ included in } z_i^\pi\}| 
        < |\calD_{s,a}| + \frac{\epsilon_{est}}{|\calS||\calA|H}F \tau
    \end{align*}
\end{lemma}
\begin{proof}
    We just need to show that under this condition, $\sum_{i=1}^F |\{(s,a)\text{ or } Fail(s,a,i) \text{ included in } z_i^\pi\}| < \frac{\tau\epsilon_{est}}{|\calS||\calA|H}F $. To show this, we use the same method and notation used in the proof of Lemma~\ref{lem (EstAll sec): doubling expected visiting time} and get that with probability at least $1 - \frac{\delta_{est}}{2|\Pi|}$,
    \begin{align*}
        &\sum_{i=1}^F |\{(s,a)\text{ or } Fail(s,a,i) \text{ included in } z_i^\pi\}| \\
        & \leq \frac{3}{2} F\mu^\pi(s,a) + \frac{\epsilon_{est}}{2|\calS||\calA|H}F 
        < \frac{2\epsilon_{est}}{|\calS||\calA|H}F 
        < \frac{\tau\epsilon_{est}}{|\calS||\calA|H}F 
    \end{align*}
\end{proof}

\begin{lemma}
\label{lem (EstAll sec): condition when add new pi}
Let $\Pi_\calD$ be the set of policies maintained before executing line~\ref{line (EstAll sec): add pi} and let $\hat{\Pi}_\calD$ be the set of policies maintained after executing. Let $(s,a)$ be the state action pair where the \textit{Fail} occurs. Then we have, with probability at least $1- \delta_{est}$,
\begin{align*}
    \mu_{\max}^{\hat{\Pi}_\calD}
    \geq \max\{2\mu_{\max}^{\Pi_\calD}(s,a),\frac{\epsilon_{est}}{|\calS||\calA|H} \}
\end{align*}
\end{lemma}
\begin{proof}
     If $\mu_{max}^{\Pi_\calD} < \frac{\epsilon_{est}}{|\calS||\calA|H}$, by Lemma~\ref{lem (EstAll sec): raise the expected visited number to delta_sim}, we know that with probability at least $1- \frac{\delta_{est}}{|\Pi|}$, we always have $\mu_{max}^{\hat{\Pi}_\calD} \geq \frac{\epsilon_{est}}{|\calS||\calA|H}$. Otherwise, if we already have $\mu_{max}^{\Pi_\calD} \geq \frac{\epsilon_{est}}{|\calS||\calA|H}$, then by Lemma~\ref{lem (EstAll sec): doubling expected visiting time}, we know that with probability at $1- \frac{\delta_{est}}{|\Pi|}$, $\mu_{max}^{\hat{\Pi}_\calD} \geq 2\mu_{max}^{\Pi_\calD} $. Finally, we take the union bound over all policies in $\Pi$ to get the target result.
\end{proof}

\subsubsection{Auxiliary Lemma}

\begin{definition}
Define $q^{\pi}_{P}(s,h)$ as the probability that policy $\pi$ will visit $s$ at step $h$ given the underlying transition probability $P$. Also define $V^{\pi}_{M}(s_1)$ as the value function that policy $\pi$ will induce given the underlaying MDP $M$.
\end{definition}

The change of the visiting probability and the value function for any fixed $\pi$ can be upper bounded in terms of the change of transition functions and expected rewards. Here we consider the most general case that the transition function and the expected rewards is non-stationary between each layers. We want to remark that, although our underlying MDP is stationary by assumption, our corruptions is allowed to be non-stationary. Also our algorithm will simulate a trajectory by the sample collected from different times. Therefore, we prove the following lemma for the non-stationary case.

\begin{lemma} [Corruption Effects on Visiting Probability ]
\label{lem (EstAll sec): transition corruption effect visiting prob}
    For any step $h'$,
    \begin{align*}
        &\sum_{s \in \calS} |q^{\pi}_{P_1}(s,h') - q^{\pi}_{P_2}(s,h')| \\
        &\leq \min \{1, \sum_{h=2}^{h'-1}\sup_{s \in \calS,a \in \calA} \|P_1(\cdot|s,a,h) - P_2(\cdot|s,a,h)\|_1 + \sup_{a \in \calA} \|P_1(\cdot|s_0,a,1) - P_2(\cdot|s_0,a,1)\|_1\}
    \end{align*}
\end{lemma}
\begin{proof}
We prove this by induction. First, we can easily get the base case that
\begin{align*}
    \sum_{s \in \calS}|q^{\pi}_{P_1}(s,2) - q^{\pi}_{P_2}(s,2)| \leq \sup_{a \in \calA} \|P_1(\cdot|s_0,a) - P_2(\cdot|s_0,a)\|_1\}.
\end{align*} 
Then by assuming that, for any step $h'\geq 3$,
\begin{align*}
    &\sum_{s \in \calS} |q^{\pi}_{P_1}(s,h') - q^{\pi}_{P_2}(s,h')| \\
    &\leq \sum_{h=2}^{h'-1}\sup_{s \in \calS,a \in \calA} \|P_1(\cdot|s,a,h) - P_2(\cdot|s,a,h)\|_1 + \sup_{a \in \calA} \|P_1(\cdot|s_0,a,1) - P_2(\cdot|s_0,a,1)\|_1,
\end{align*}
we have that, for any step $h'+1$,
\begin{align*}
    &\sum_{s \in \calS}|q^{\pi}_{P_1}(s,h'+1) - q^{\pi}_{P_2}(s,h'+1)|\\
    & \leq \sum_{s \in \calS}|\sum_{s' \in \calS}\left( q^{\pi}_{P_1}(s',h') - q^{\pi}_{P_2}(s',h')\right) P_1(s|s',\pi_{h'}(s'),h')| \\
         & \quad + \sum_{s \in \calS} |\sum_{s' \in \calS}q^{\pi}_{P_2}(s',h')\left( P_1(s|s',\pi_{h'}(s',h',h') -  P_2(s|s',\pi_{h'}(s'),h')\right)| \\
    &\leq \sum_{s' \in \calS}\left|q^{\pi}_{P_1}(s',h') - q^{\pi}_{P_2}(s',h')\right| \sum_{s \in \calS}P_1(s|s',\pi_{h'}(s'))
         +  \sum_{s' \in \calS}q^{\pi}_{P_2}(s',h')\sum_{s \in \calS}\left| P_1(s|s',\pi(s',h') -  P_2(s|s',\pi_{h'}(s'))\right| \\
    &\leq \sum_{s' \in \calS}\left|q^{\pi}_{P_1}(s',h') - q^{\pi}_{P_2}(s'.h')\right|
        + \sup_{s' \in \calS} \sum_{s \in \calS}\left| P_1(s|s',\pi_{h'}(s'),h') -  P_2(s|s',\pi_{h'}(s'),h')\right|\\
    &\leq  \sum_{h=2}^{h'}\sup_{s \in \calS,a \in \calA} \|P_1(\cdot|s,a) - P_2(\cdot|s,a)\|_1 + \sup_{a \in \calA} \|P_1(\cdot|s_0,a,h') - P_2(\cdot|s_0,a,h')\|_1
\end{align*}
\end{proof}

\begin{lemma}[Corruption effects on value function ]
\label{lem (EstAll sec): transition corruption effect value function}
\begin{align*}
    |V^{M_1,\pi} -V^{M_2,\pi}| 
    & \leq H\sum_{h=2}^{H} \sup_{s' \in \calS} \|  P_1(\cdot|s',\pi(s'),h) -  P_2(\cdot|s',\pi(s'),h)\|_1  
        +  \sum_{h=2}^H\sup_{s \in \calS} |\mu_1(s,\pi(s),) - \mu_2(s,\pi(s),h) | \\
        &\quad + \|  P_1(\cdot|s_0,\pi(s_0),1) -  P_2(\cdot|s_0,\pi(s_0),1)\|_1  
        + |\mu_1(s_0,\pi(s_0),1) - \mu_2(s_0,\pi(s_0),1) | 
\end{align*}
\end{lemma}
\begin{proof}
For convenience, when I write $\sum_{h=1}^H\sum_{s \in \calS}$ in the following, I actually mean $\sum_{h=2}^H\sum_{s \in \calS} + \sum_{s = s_0}$.
\begin{align*}
    & |V^{M_1,\pi}(s_0) -V^{M_2,\pi}(s_0)| \\
    & \leq |\sum_{h=1}^H\sum_{s \in \calS} \left( q^\pi_{P_1}(s,h) - q^\pi_{P_2}(s,h)\right)\mu_1(s,\pi(s),h)|
        + |\sum_{h=1}^H\sum_{s \in \calS} q^\pi_{P_2}(s,h)\left(\mu_1(s,\pi(s),h) - \mu_2(s,\pi(s),h)\right)\  | \\
    & \leq |\sum_{h=1}^H \sup_{s \in \calS}\mu_1(s,\pi_1(s)) \sum_{s \in \calS} \left( q^\pi_{P_1}(s,h) - q^\pi_{P_2}(s,h)\right)  |
        + \sum_{h=1}^H\sup_{s \in \calS} |\mu_1(s,\pi(s),h) - \mu_2(s,\pi(s),h) |\\
    & \leq \left( \sum_{h=1}^H \sup_{s \in \calS}\mu_1(s,\pi_1(s)) \right) \left(\sum_{h=1}^{H} \sup_{s \in \calS} \|  P_1(\cdot|s,\pi(s),h) -  P_2(\cdot|s,\pi(s),h)\|_1  \right)\\
          &\quad  + \sum_{h=1}^H\sup_{s \in \calS} |\mu_1(s,\pi(s),h) - \mu_2(s,\pi(s),h) |\\
    &\leq H\sum_{h=1}^{H} \sup_{s \in \calS_{h}} \|  P_1(\cdot|s,\pi(s),h) -  P_2(\cdot|s,\pi(s),h)\|_1  
        +   \sum_{h=1}^H\sup_{s \in \calS} |\mu_1(s,\pi(s),h) - \mu_2(s,\pi(s),h) |
\end{align*}
Here the third inequality comes from Lemma~\ref{lem (EstAll sec): transition corruption effect visiting prob} and the last inequality comes from the assumption on the reward function.
\end{proof} 
\section{Discussion on Reward-free Exploration Algorithm under Corruptions}
\label{sec: app-RFalgo}

In the Related Work section, we mentioned that algorithms proposed in \citet{kaufmann2020adaptive} and \citet{menard2020fast} can \textit{efficiently} achieve uniform $\epsilon$-close estimations for all the polices with near-optimal sample complexity in the no-corruption setting. Their main idea is to construct a computable estimator of Q-value estimation error for all the state-action pairs and greedily play the action that maximize such estimator at every step until all the state-action pairs have sufficiently small Q-value estimation errors. So a natural question to ask is, 
\begin{center}
    Can we replace the \textsc{EstAll} with this type of efficient algorithms ? 
\end{center}
To be specific, firstly, in the non-corrupted setting, we want to find an efficient algorithm that can guarantee uniform estimations on all the policies in any given policy set $\Pi$ by only implementing polices inside $\Pi$. Secondly, we also want this algorithm has corruption robustness at least not worse than the \textsc{EstAll}.

For the first target, we can easily define an estimator $W_t(\pi) = \sum_{h=1}^H \sum_{s \in \calS} \frac{\hat{p}_{t,h}^\pi(s)}{n_h^t(s,\pi(s))}$, where $n_h^t(s,\pi(s))$ is the empirical number of times state-action-step pair $(s,\pi(s),h)$ has been visited before time $t+1$ and $\hat{p}_{t,h}^\pi(s)$ is the empirical probability that the policy $\pi$ reach state $s$ at $h$ before time $t+1$. Suppose we have an efficient oracle that can calculate the following in the polynomial times,
\begin{align*}
    \argmax_{\pi \in \Pi} W_t(\pi)
\end{align*}
Then we can find an oracle-efficient algorithm by greedily sampling $\pi_{t+1} = \argmax_{\pi \in \Pi} W_t(\pi)$ until all the $W_t(\pi)$ are small enough. 

Unfortunately, in the presence of corruptions, we find it is hard to get a good robustness. Roughly speaking, suppose the rewards are fixed, then the estimation error $ \hat{V}^\pi$ for any policy $\pi$ is upper bounded by 
\begin{align*}
    |V^\pi - \hat{V}^\pi| \leq \min_{t \in \calI} C_{\calI}^p W_t(\pi_{t+1}) + \sqrt{W_t(\pi_{t+1})}
\end{align*}
where $\calI$ represents the whole time period this algorithm is running. Then from our perspective, when $|\calI| = o(1/\epsilon^2)$, we can only guarantee $\min_{t \in \calI} W_t(\pi_{t+1}) \leq \Tilde{\order}\left(poly(|\calS||\calA|H (\epsilon^2 + C_\calI^p\epsilon^2)\right)$, which gives
\begin{align*}
    |V^\pi - \hat{V}^\pi| \leq \Tilde{\order}\left(poly(|\calS||\calA|H) ((C_\calI^p)^2\epsilon^2 + \sqrt{C_\calI^p}\epsilon) \right)
\end{align*}
Note that \textsc{EstAll} gives $\Tilde{\order}\left(poly(|\calS||\calA|H ((C_\calI^p)^2\epsilon^2 +\epsilon)) \right)$-close estimations when $C_\calI^p \leq 1/\epsilon$. Therefore, plug-in this algorithm instead of \textsc{EstAll} in \textsc{BARBAR-RL} will give worse dependence in $T$.

\begin{center}
    \textbf{Whether we can find a better estimator in this type of reward-free sub-algorithms or whether we can find another proper meta-algorithm for this type of sub-algorithms remains open.}
\end{center}

% By some simple modification modification, we think the first target should can achieve the following theorem, 
% \begin{theorem}[Informal]
% In the non-corrupted setting, suppose we have an oracle-efficient algorithm 
% \end{theorem}

\iffalse
\section{Do \emph{not} have an appendix here}

\textbf{\emph{Do not put content after the references.}}
%
Put anything that you might normally include after the references in a separate
supplementary file.

We recommend that you build supplementary material in a separate document.
If you must create one PDF and cut it up, please be careful to use a tool that
doesn't alter the margins, and that doesn't aggressively rewrite the PDF file.
pdftk usually works fine. 

\textbf{Please do not use Apple's preview to cut off supplementary material.} In
previous years it has altered margins, and created headaches at the camera-ready
stage. 
\fi
%%%%%%%%%%%%%%%%%%%%%%%%%%%%%%%%%%%%%%%%%%%%%%%%%%%%%%%%%%%%%%%%%%%%%%%%%%%%%%%
%%%%%%%%%%%%%%%%%%%%%%%%%%%%%%%%%%%%%%%%%%%%%%%%%%%%%%%%%%%%%%%%%%%%%%%%%%%%%%%

\end{document}